\documentclass[11pt]{article}

\usepackage{graphicx,amsmath,amsfonts,amssymb,bm,hyperref,url,breakurl,epsfig,epsf,color,fullpage,MnSymbol,mathbbol,fmtcount,algpseudocode,algorithm,semtrans,cite,caption,subcaption,multirow,xcolor}

  \usepackage{mathtools}

\usepackage{pbox}

\usepackage{titlesec}
\usepackage{booktabs}
\usepackage{pifont}

\usepackage{multirow}

\usepackage{tikz}
\usepackage{pgfplots}
\usetikzlibrary{pgfplots.groupplots}
\titleformat{\paragraph}
{\normalfont\normalsize\bfseries}{\theparagraph}{1em}{}
\titlespacing*{\paragraph}
{0pt}{3.25ex plus 1ex minus .2ex}{1.5ex plus .2ex}

\usepackage{movie15}

\usepackage{cite}
\usepackage[bottom,hang,flushmargin]{footmisc} 

\usepackage[font=footnotesize,labelfont=bf]{caption}

\usepackage{hyperref}
\definecolor{darkred}{RGB}{150,0,0}
\definecolor{darkgreen}{RGB}{0,150,0}
\definecolor{darkblue}{RGB}{0,0,200}
\hypersetup{colorlinks=true, linkcolor=darkred, citecolor=darkgreen, urlcolor=darkblue}

\newtheorem{theorem}{Theorem}[section]
\newtheorem{lemma}[theorem]{Lemma}

\newtheorem{definition}[theorem]{Definition}

\newcommand{\beq}{\begin{equation}}
\newcommand{\eeq}{\end{equation}}
\newcommand{\prox}{{{\text{\bf{prox}}}}}
\newcommand{\nn}{\nonumber}
\newcommand{\la}{\lambda}

\newcommand{\Iden}{{\mtx{I}}}

\newcommand{\z}{{\mtx{z}}}
\newcommand{\dist}{\text{dist}}
\newcommand{\paf}{\partial f(\x)}
\newcommand{\Cc}{\mathcal{C}}
\newcommand{\Bc}{\mathcal{B}}

\newcommand{\Dc}{\mathcal{D}}

\newcommand{\Pc}{\mathcal{P}}

\newcommand{\Nn}{\mathcal{N}}

\newcommand{\vb}{{\mtx{v}}}
\newcommand{\ub}{\mtx{u}}
\newcommand{\rng}{{\psi_\tau}}

\newcommand{\w}{\mtx{w}}

\newcommand{\li}{\left<}
\newcommand{\ri}{\right>}

\newcommand{\ab}{\vct{a}}

\newcommand{\h}{\vct{h}}
\newcommand{\g}{\vct{g}}

\newcommand{\fronorm}[1]{\left\|#1\right\|_{F}}
\newcommand{\onenorm}[1]{\left\|#1\right\|_{\ell_1}}
\newcommand{\twonorm}[1]{\left\|#1\right\|_{\ell_2}}
\newcommand{\tn}[1]{\left\|#1\right\|_{\ell_2}}

\newcommand{\abs}[1]{\left|#1\right|}

\newcommand{\comment}[1]{}

\newcommand{\x}{\vct{x}}
\newcommand{\y}{\vct{y}}

\newcommand{\thets}{\vct{\theta}^*}

\definecolor{emmanuel}{RGB}{255,127,0}

\newcommand{\R}{\mathbb{R}}
\newcommand{\Pro}{\mathbb{P}}

\newcommand{\sgn}[1]{\textrm{sgn}(#1)}

\newcommand{\E}{\operatorname{\mathbb{E}}}

\newcommand{\e}{\mathrm{e}}

\newcommand{\vct}[1]{\bm{#1}}
\newcommand{\mtx}[1]{\bm{#1}}

\newcommand{\ns}{{\cal R}}
\newcommand{\Glf}{\mathcal{G}(\la)}

\newcommand{\Xp}{\mtx{X}_{\perp}}
\newcommand{\X}{{\mtx{X}}}
\newcommand{\bbeta}{{\boldsymbol{\beta}}}

\numberwithin{equation}{section} 

\def \endprf{\hfill {\vrule height6pt width6pt depth0pt}\medskip}
\newenvironment{proof}{\noindent {\bf Proof} }{\endprf\par}

\title{Fast and Reliable Parameter Estimation\\
from Nonlinear Observations}
\author{Samet Oymak\thanks{Google Inc. 1600 Amphitheatre Parkway, Mountain View, CA}\quad and\quad Mahdi
  Soltanolkotabi\thanks{Ming Hsieh Department of Electrical Engineering, University of Southern California, Los Angeles, CA} }
\date{October 2016}

\begin{document}

\maketitle

\begin{abstract} 
In this paper we study the problem of recovering a structured but unknown parameter $\vct{\theta}^*$ from $n$ nonlinear observations of the form $y_i=f(\langle \vct{x}_i,\vct{\theta}^*\rangle)$ for $i=1,2,\ldots,n$. We develop a framework for characterizing time-data tradeoffs for a variety of parameter estimation algorithms when the nonlinear function $f$ is unknown. This framework includes many popular heuristics such as projected/proximal gradient descent and stochastic schemes. For example, we show that a projected gradient descent scheme converges at a linear rate to a reliable solution with a near minimal number of samples. We provide a sharp characterization of the convergence rate of such algorithms as a function of sample size, amount of a-prior knowledge available about the parameter and a measure of the nonlinearity of the function $f$. These results provide a precise understanding of the various tradeoffs involved between statistical and computational resources as well as a-prior side information available for such nonlinear parameter estimation problems.

%We consider a general model in which the aim is to estimate an unknown parameter $\bbeta$ from noisy and possibly nonlinear observations $f(\X\bbeta)$. In statistics this is known as the single index model and it encompasses several models including lasso, logistic regression and more recently it finds applications in one-bit compressed sensing. In this work, we introduce fast iterative algorithms that achieves linear convergence to the underlying parameter. In important instances including projected gradient descent, we achieve optimal estimation error rates in terms of the nonlinearity and the problem geometry.
\end{abstract}
\section{Introduction}

Parameter estimation is fundamental to many supervised learning tasks in signal processing and machine learning. Given training data consisting of $n$ pairs of input features (a.k.a. measurements) $\vct{x}_i\in\R^p$ and desired outputs $\vct{y}_i\in\R$ we wish to infer a function that best explains the training data. The simplest functions are linear ones where the outputs are linear functions of the features $y_i=\langle\vct{x}_i,\vct{\theta}^*\rangle$ with $\vct{\theta}^*\in\R^p$ an unknown parameter to be learned from data. Let $\mtx{X}\in\R^{n\times p}$ be a feature matrix with rows containing the $n$ features $\vct{x}_1,\vct{x}_2,\ldots,\vct{x}_n$ and the vector $\vct{y}\in\R^n$ containing $n$ output values $y_1, y_2, \ldots, y_n$. The parameter $\vct{\theta}^*$ is typically estimated by solving an optimization problem of the form
\begin{align}
\label{main}
\underset{\vct{\theta}\in\R^p}{\min}\quad\frac{1}{2}\twonorm{\vct{y}-\mtx{X}\vct{\theta}}^2\quad\text{subject to}\quad \mathcal{R}(\vct{\theta})\le R.
\end{align}
Here, $\mathcal{R}(\vct{\theta})$ is a regularization function used to avoid overfitting specially when the number of samples $n$ is significantly smaller than the number of parameters $p$. For example when the parameter $\vct{\theta}^*$ is believed to be sparse a typical regularization function is $\mathcal{R}(\vct{\theta})=\onenorm{\vct{\theta}}$. Over the last few years there has been significant progress in understanding the properties of the optimization problem \eqref{main} and when it is successful in recovering the unknown parameter $\vct{\theta}^*$ and in turn predicting future outcomes given a new feature vector. We shall review some of this literature in Section \ref{sec:priorart}.

Even though linear regression models are widely used they rarely capture the feature-output relation in the data precisely. For example in signal processing, such linear models are often first order approximation of typically unknown nonlinear mappings. In this paper we study the sensitivity of various iterative shrinkage schemes used for solving \eqref{main} to such modeling mismatch. More specifically, we assume that the output is related to the features via the nonlinear equations
\begin{align}
\label{nonlineq}
y_i=f(\langle\vct{x}_i,\vct{\theta}^*\rangle)\quad \text{for } i=1,2,\ldots,n.
\end{align}
Here, $f:\R\rightarrow\R$ is an unknown function and $\vct{\theta}^*\in\R^p$ is the unknown parameter we wish to estimate. This model is also known as the \emph{single index model} in the statistics literature. Throughout this paper we shall use $\vct{y}=f(\mtx{X}\vct{\theta}^*)$ as a shorthand for \eqref{nonlineq}. With $f$ unknown, it is natural to try to find the unknown parameter $\vct{\theta}^*$ via the optimization problem \eqref{main}. In this paper we study the effectiveness of various optimization heuristics used for solving \eqref{main} under this nonlinear modeling assumption. Our results are very general and hold even when the regularization function $\mathcal{R}$ is nonconvex. This is perhaps surprising as it is not clear that when $\mathcal{R}$ is nonconvex the global optimum to \eqref{main} can be found. We precisely characterize the run time of projected/proximal gradient and stochastic gradient methods for solving such problems as a function of the number of outputs, the ability of the function $\mathcal{R}$ to enforce prior information and a measure of the nonlinearity of the function. Our results provide an accurate understanding of the various computational and statistical tradeoffs involved when solving such nonlinear parameter estimation problems.

\section{Precise measures for statistical resources}
To arrive at precise tradeoffs between computational and statistical resources we need to quantify the various resources. Computation is easily measured in terms of time/iterations. We measure data size or \emph{sample complexity} in terms of the number of samples $n$. Naturally the required number of samples for reliable parameter estimation depends on how well the regularization function $\mathcal{R}$ can capture the properties of the underlying parameter $\vct{\theta}^*$. For example if we know our unknown parameter is approximately sparse naturally using an $\ell_1$ norm for the regularizer is superior to using an $\ell_2$ regularizer. To quantify this capability we first need a couple of standard definitions which we adapt from \cite{oymak2015sharp}.
\begin{definition}[Descent set and cone] \label{decsetcone} The \emph{set of descent} of  a function $t$ at a point $\vct{\theta}^*$ is defined as
\begin{align*}%\text{for some } c\geq 0
{\cal D}_{\mathcal{R}}(\vct{\theta}^*)=\Big\{\vct{h}:\text{ }\mathcal{R}(\vct{\theta}^*+\vct{h})\le \mathcal{R}(\vct{\theta}^*)\Big\}.
\end{align*}
The \emph{cone of descent} is defined as a closed cone $\mathcal{C}_{\mathcal{R}}(\vct{\theta}^*)$ that contains the descent set, i.e.~$\mathcal{D}_{\mathcal{R}}(\vct{\theta}^*)\subset\mathcal{C}_{\mathcal{R}}(\vct{\theta}^*)$. The \emph{tangent cone} is the conic hull of the descent set. That is, the smallest closed cone $\mathcal{C}_{\mathcal{R}}(\vct{\theta})$ obeying $\mathcal{D}_{\mathcal{R}}(\vct{\theta}^*)\subset\mathcal{C}_{\mathcal{R}}(\vct{\theta}^*)$.
\end{definition}
We note that the capability of the regularizer $\mathcal{R}$ in capturing the properties of the unknown parameter $\vct{\theta}^*$ depends on the size of the descent cone $\mathcal{C}_{\mathcal{R}}(\vct{\theta}^*)$. The smaller this cone is the more suited the function $\mathcal{R}$ is at capturing the properties of $\vct{\theta}^*$. To quantify the size of this set we shall use the notion of mean-width.
\begin{definition}[Gaussian width] The Gaussian width of a set $\mathcal{C}\in\R^p$ is defined as:
\begin{align*}
\omega(\mathcal{C}):=\mathbb{E}_{\vct{g}}[\underset{\vct{z}\in\mathcal{C}}{\sup}~\langle \vct{g},\vct{z}\rangle],
\end{align*}
where the expectation is taken over $\vct{g}\sim\mathcal{N}(\vct{0},\mtx{I}_p)$.
\end{definition}
We now have all the definitions in place to quantify the capability of the function $\mathcal{R}$ in capturing the properties of the unknown parameter $\vct{\theta}^*$. This naturally leads us to the definition of the minimum required number of samples.
\begin{definition}[minimal number of samples]\label{PTcurve}
Let $\mathcal{C}_{\mathcal{R}}(\vct{\theta}^*)$ be a cone of descent of $\mathcal{R}$ at $\vct{\theta}^*$ and set $\omega=\omega(\mathcal{C}_{\mathcal{R}}(\vct{\theta}^*)\cap\mathcal{B}^{n})$. Also let $\phi(t)=\sqrt{2}\frac{\Gamma(\frac{t+1}{2})}{\Gamma(\frac{t}{2})}\approx \sqrt{t}$. We define the minimal sample function as
\begin{align*}
\mathcal{M}(\mathcal{R},\vct{\theta}^*,t)=\phi^{-1}(\omega+t)\approx (\omega+t)^2.
\end{align*}
We shall often use the short hand $n_0=\mathcal{M}(\mathcal{R},\vct{\theta}^*,t)$ with the dependence on $\mathcal{R},\vct{\theta}^*,t$ implied. We note that for convex functions $\mathcal{R}$ based on \cite{Cha, McCoy} $n_0$ is exactly the minimum number of samples required for the estimator \eqref{main} to succeed in recovering an unknown parameter $\vct{\theta}^*$ with high probability from linear measurements $\vct{y}=\mtx{X}\vct{\theta}^*$. With some overloading, even for non-convex functions $\mathcal{R}$, we shall refer to $n_0$ as the ``minimum number of samples".   
\end{definition}
We pause to note that prior literature \cite{OymLin, McCoy} indeed shows that $n_0$ is a good notion of complexity demonstrating that when the feature-response relationship is linear (i.e.~$f$ is a linear map) and the regularizer $\mathcal{R}$ is convex the properties of the estimator \eqref{main} can be precisely characterized in terms of $n_0$. Please also see \cite{vershynin2015estimation, bruer2014time, oymak2015sharp} for some extensions to the non-convex case as well as the role this quantity plays in the computational complexity of projected gradient schemes.

Finally, to analyze algorithm performance as a function of the nonlinearity of the map $f$, we shall use three parameters. Two of these parameters are essentially the intrinsic mean and variance associated with the nonlinear map $f(\cdot)$ and the final term captures a non-asymptotic deviation from linearity. 
\begin{definition}[nonlinearity parameters]\label{nonlin} Let $g\in\R$ be a standard normal random variable. Define,\begin{itemize}
\item {\bf{Mean term:}} $\mu=\E[f(g)g]$,
\item {\bf{Variance term:}} $\sigma^2=\E[(f(g)-\mu g)^2]$,
\item {\bf{Deviation term:}} $\gamma^2=\E[g^2(f(g)-\mu g)^2]$.
\end{itemize}
\end{definition}
To see that these nonlinearity measures conform with our intuition, note that in the linear case $f(\mtx{X}\vct{\theta})=\mtx{X}\vct{\theta}$ so that one can take $f(g)=g$ and so the mean term is equal to one and the variance and deviation terms are equal to zero. When $f$ is a nonlinear function like $f(g)=\sgn{g}$ we have $\mu=\sqrt{2/\pi}<1$ and $\sigma^2=\gamma=1-2/\pi>0$. While these examples are beneficial to give intuition, we shall see that for fast and reliable estimation of the vector $\vct{\theta}^*$ we do not require explicit knowledge of these parameter values.

We can view nonlinear measurements of the form $\vct{y}=f(\mtx{X}\vct{\theta}^*)$ as linear measure from the vector $\mu \vct{\theta}^*$ with an effective noise term $\vct{w}=f(\X\vct{\theta}^*)-\mu\X\vct{\theta}^*$. Intuitively, when the features $\mtx{X}$ are sufficiently randomized, we expect this effective noise to behave similar to $f(\vct{g})-\mu\vct{g}$ where $\vct{g}\sim\mathcal{N}(\vct{0},\mtx{I}_n)$. Consequently, we expect $\E[\tn{\w}^2]\propto n\sigma^2$, which further justifies the definition of $\sigma^2$ as a sort of ``variance" from nonlinearity. However, control of this variance term is not enough to show reliable estimation of the parameter with high probability. To be able to make probabilistic statements, we need the Euclidean norm of the effective noise ($\tn{\w}^2$) to concentrate around its expected value. In particular when $\vct{w}=f(\vct{g})-\mu\vct{g}$, the quantity $\tn{\w}^2$ is the sum of $n$ i.i.d.~random variables and exponentially concentrates under mild conditions. For simplicity we will state our results in terms of the concentration probability function defined below.
\begin{definition}[Concentration probability function]\label{defprob} Assume $\vct{g}\in\R^n$ is distributed as a standard normal random vector $\mathcal{N}(\vct{0},\mtx{I}_n)$ and set $b_n=\E[\twonorm{\vct{g}}]\approx \sqrt{n}$. We define the concentration probability function as
\beq
p(\eta)=\Pro(\tn{f(\g)-\mu\g}>\eta b_n\sigma)+\Pro(|\g^T(f(\g)-\mu\g)|>\eta \frac{b_n^2}{\sqrt{n}}\gamma),\nn
\eeq
with $\mu$, $\sigma$, and $\gamma$ as defined in Definition \ref{nonlin}. 
\end{definition}
We note that Markov inequality implies $p(\eta)\leq \frac{2}{\eta^2}\frac{n^2}{b_n^4}\approx \frac{2}{\eta^2}$. However, for many nonlinear functions one can obtain much sharper concentration bounds. For example, when the function $f$ is bounded or Lipschitz, standard concentration of sub-exponential random variables show that $p(\eta)\leq e^{-c\min(\eta,\eta^2)n}$ with $c$ a constant depending only on the upper bound/Lipshitz constant of the function $f$. Now that we have precise measures for the various statistical resources we are ready to state our results. 

\section{Precise convergence rates for iterative shrinkage schemes}
As mentioned earlier we wish to understand the convergence rates of different iterative shrinkage schemes used for solving nonlinear parameter estimation problems. Throughout this paper we assume that the features $\vct{x}_i$ are i.i.d.~random Gaussian vectors with distribution $\mathcal{N}(\vct{0},\mtx{I})$. Furthermore, without loss of generality we assume the unknown parameter $\vct{\theta}^*$ has unit Euclidean norm.

\subsection{Projected Gradient Descent}
\label{secPGD}
Perhaps the simplest algorithm for solving the nonlinear equations \eqref{nonlineq} is \emph{Project Gradient Descent (PGD)}  where we use gradient descent on the least squares cost 
\begin{align*}
\frac{1}{2}\twonorm{\vct{y}-\mtx{X}\vct{\theta}}^2,
\end{align*}
 followed by projection on the constraint set $\mathcal{K}=\{\vct{\theta}\in\R^p: \mathcal{R}(\vct{z})\le R\}$. More specifically, starting from an initial vector $\vct{\theta}_0$, PGD iteratively applies the update
\begin{align}
\label{PGDiter}
\vct{\theta}_{\tau+1}=\mathcal{P}_{\mathcal{K}}\left(\vct{\theta}_\tau+\alpha_\tau \mtx{X}^T(\vct{y}-\mtx{X}\vct{\theta}_\tau)\right).
\end{align}
Here, $\mathcal{P}_{\mathcal{K}}(\vct{z})$ denotes the Euclidean projection of the vector $\vct{z}$ onto the set $\mathcal{K}$ and $\alpha_\tau$ is the step size. Throughout this paper we will assume that the tuning parameter is perfectly tuned so that $R=\mathcal{R}(\mu\vct{\theta}^*)$ where $\mu$ is the mean term per Definition \ref{nonlin}. However, we can extend our arguments to the case where $R\neq \mathcal{R}(\mu\vct{\theta}^*)$ by utilizing the sensitivity analysis developed in \cite[Theorem 2.6]{oymak2015sharp}.

Our first result shows that projected gradient descent allows for fast and reliable parameter estimation from nonlinear observations.
\begin{theorem}\label{PGDthm} Let $f:\R\rightarrow \R$ be an unknown nonlinear function. Also, let $\vct{y}=f(\mtx{X}\vct{\theta}^*)\in\R^n$ be $n$ nonlinear observations from $\vct{\theta}^*$ with the feature matrix $\mtx{X}\in\R^{n\times p}$ consisting of i.i.d. $\mathcal{N}(0,1)$ entries. Furthermore, let $\kappa_{\mathcal{R}}$ be a constant that is equal to one for convex regularizers $\mathcal{R}$ and equal to two for nonconvex ones. Furthermore, let $\vct{n}_0=\mathcal{M}(\mathcal{R},\mu\vct{\theta}^*,t)$ be the minimal number of data samples as per Definition \ref{PTcurve} and assume
\begin{align}
\label{PGDsampeq}
n\ge 8\kappa_{\mathcal{R}}^2 n_0.
\end{align}
Also let $\mu, \sigma$, and $\eta$ be the nonlinearity parameters per Definition \ref{nonlin}. Then, starting from any initial estimate $\vct{\theta}_0$ the PGD iterates \eqref{PGDiter} with step size $\alpha_\tau=1/b_n^2\approx 1/n$ and tuning parameter $R=\mathcal{R}(\mu\vct{\theta}^*)$ obeys
\begin{align}
\label{PGDconveq}
\twonorm{\vct{\theta}_\tau-\mu\vct{\theta}_0}\le \left(\sqrt{8\kappa_{\mathcal{R}}^2\frac{n_0}{n}}\right)^\tau\twonorm{\vct{\theta}_0-\mu\vct{\theta}^*}+\frac{\kappa_{\mathcal{R}}}{1-\left(\sqrt{8\kappa_{\mathcal{R}}^2\frac{n_0}{n}}\right)}\frac{\eta\left(\sigma \sqrt{n_0}+\gamma\right)}{\sqrt{n}},
\end{align}
for all $\tau$ with probability at least $1-p(\eta)-10e^{-\frac{t^2}{8}}$.~Here, $p(\eta)$ is the concentration probability function as per Definition \ref{defprob}.
\end{theorem}

Note that $\eta\left(\sigma \sqrt{n}\right)$ is roughy the Euclidean norm of the effective noise $\vct{w}=f(\mtx{X}\vct{\theta}^*)-\mu\mtx{X}\vct{\theta}^*$ induced by replacing the nonlinear equation $\vct{y}=f(\mtx{X}\vct{\theta}^*)$ with the linear equation $\vct{y}=\mu\mtx{X}\vct{\theta}^*$. Thus, the theorem above shows that the projected gradient updates converge at a linear rate to a small neighborhood around the ``true" solution $\mu \vct{\theta}^*$. The radius of this neighborhood decreases with an increase in the number of samples $n$. The size of this radius is near-optimal and comparable to recent results \cite{thrampoulidis2015lasso,plan2015generalized,negahban2009unified} where the estimate is obtained by solving the convex optimization problem in \eqref{main} which applies only when the regularization function $\mathcal{R}$ is convex. Indeed, the size of this radius scales like $\sqrt{n_0/n}\twonorm{\vct{w}}$ which up to a small constant is exactly the same as the result one would get when the model is linear of the form $\y=\mu\X\theta^*+\vct{w}$. This is perhaps unexpected as it demonstrates that our nonlinear model exactly behaves like a fictitious linear model with the same effective noise!

The convergence rate in \eqref{PGDconveq} is linear and proportional to $1/\sqrt{n}$. This shows that the convergence rate improves with an increase in the sample size which in turn leads to faster convergence with more data samples. This implies that the more samples we have not only does the quality of the solution improve but also that we arrive at this solution with less computational effort. Thus, more samples is beneficial both in terms of statistical reliability and computational efficiency.\footnote{We note that the latter statement about computational efficiency is true only to a certain extent. In particular when the feature matrix $\mtx{X}$ does not have a fast vector-matrix multiply the improvemed convergence rate is soon dominated by the increased cost of the matrix-vector multiplies in each iteration due to the increase in the number of samples and hence the size of the feature matrix. However, we expect \eqref{PGDconveq} to be correct for many random feature models that do have fast vector-matrix multiply e.g. Fourier type matrices in signal processing or sparse features in machine learning.}

Finally, another interesting aspect of Theorem \ref{PGDthm} is that it applies to both convex and nonconvex regularizers. Showing that one can obtain statistically reliable solutions with nonconvex regularizers is perhaps unexpected as it is not a-priori clear that the global solution to \eqref{main} can be found using a computationally tractable algorithm.\footnote{This statement is of course only true if the projection onto the sub-level sets of the regularization function is computationally efficient. This is true for many non-convex regularizers including some $\ell_p$ norms with $p<1$. Furthermore, we note that projection onto convex sets may also not be in general tractable a good example of this is projection onto the set of completely positive matrices which is known to be NP-hard.}
%for the updates starting from some initial solution $\vct{\theta}_0$ we iteratively run

\subsection{Stochastic gradient schemes}
In this section we provide convergence guarantees for stochastic gradient schemes. Our result concerns a Projected variation of the Stochastic Gradient Descent algorithm \cite{bottou2010large} which we refer to as PSGD. Let $\{\rng\}_{\tau=1}^{\infty}$ denote random integer taking the value $i\in\{1,2,\ldots,n\}$ with probability $\frac{\twonorm{\vct{x}_i}^2}{\fronorm{\mtx{X}}^2}$. Starting from an initial estimate $\vct{\theta}_0$ the PSGD algorithm iteratively applies the updates
\begin{align}
\label{PSGDiter}
\vct{\theta}_{\tau+1}=\Pc_{\mathcal{K}}\left(\vct{\theta}_\tau+\frac{( \y_\rng-\langle \vct{x}_{\psi_\tau},\vct{\theta}_\tau\rangle)}{\twonorm{\vct{x}_{\psi_\tau}}^2}\vct{x}_{\psi_\tau}\right).
\end{align}
We will show that it is possible to obtain guarantees that are on par with the PGD guarantees derived in the previous section. This suggests that nonlinear parameter estimation problems can be solved by highly parallel and asynchronous algorithms. As in the previous case our error bounds will again only depend on the effective noise terms $\sigma$ and $\gamma$.
\begin{theorem}\label{PSGDthm} Let $f:\R\rightarrow \R$ be an unknown nonlinear function. Also, let $\vct{y}=f(\mtx{X}\vct{\theta}^*)\in\R^n$ be $n$ nonlinear observations from $\vct{\theta}^*$ with the feature matrix $\mtx{X}\in\R^{n\times p}$ consisting of i.i.d. $\mathcal{N}(0,1)$ entries. Also, let $\mathcal{R}$ be a convex regularizer. Furthermore, let $\vct{n}_0=\mathcal{M}(\mathcal{R},\mu\vct{\theta}^*,t)$ be the minimal number of data samples as per Definition \ref{PTcurve} and assume
\begin{align}
\label{PSGDsampeq}
n>n_0.
\end{align}
Also let $\mu, \sigma$, and $\eta$ be the nonlinearity parameters per Definition \ref{nonlin}. Then, starting from any initial estimate $\vct{\theta}_0$ the PSGD iterates \eqref{PSGDiter} with tuning parameter $R=\mathcal{R}(\mu\vct{\theta}^*)$ obey
\begin{align}
\label{PSGDconveq}
\mathbb{E}[\twonorm{\vct{\theta}_\tau-\mu\vct{\theta}^*}^2]\le \left(1-\frac{\left(1-\sqrt{\frac{n_0}{n}}\right)^2}{2p}\right)^\tau\twonorm{\vct{\theta}_0-\mu\vct{\theta}^*}^2+\frac{1.01}{\left(1-\sqrt{\frac{n_0}{n}}\right)^2}\eta^2\sigma^2,
\end{align}
for all $\tau$ with probability at least $1-(n+1)e^{-cp}-p(\eta)$. Here, the expectation is over the random variables $\{\psi_s\}_{s=1}^\tau$.
\end{theorem}
Similar to our results for the PGD iterations the results of Theorem \ref{PSGDthm} above demonstrates that PSGD also converges at a geometric rate to a neighborhood of the true solution $\mu\vct{\theta}^*$. However, comparing \eqref{PSGDconveq} with its counterpart in \eqref{PGDconveq} the PSGD results are weaker in two ways. First, while the convergence is geometric it is no longer linear and slower than its PGD counter part. However, we should point out that the cost of each iteration of PSGD is roughly $1/n$ times the cost of a PGD update. Second, the radius of the neighborhood of the true solution is larger in the PSGD case compared with the PGD case by a factor of roughly size $\sqrt{n/n_0}$. We believe this to be an artifact of our proof technique and we expect that for $\tau$ sufficiently large the the second term should be of the form $\frac{n_0}{n}\eta^2\sigma^2$ in lieu of $\eta^2\sigma^2$.

\subsection{Proximal gradient methods}
In Section \ref{secPGD} we discussed our results for nonlinear estimation problem by enforcing the constraint ${\ns}(\vct{\theta})\leq {\ns}(\mu\vct{\theta}^*)$. Another popular approach for finding structured solutions to linear inverse problems is solving a penalized variant of \eqref{lasso main}. Our framework also provides some insights for convergence of such proximal gradient methods. However, our results for proximal methods require a few additional definitions and modeling assumptions. We therefore defer these results to Appendix \ref{sec proximal}.

\section{Numerical Experiments}
\begin{figure}
\centering
\begin{subfigure}[t]{0.5\textwidth}
\centering
\includegraphics[width=0.9\linewidth]{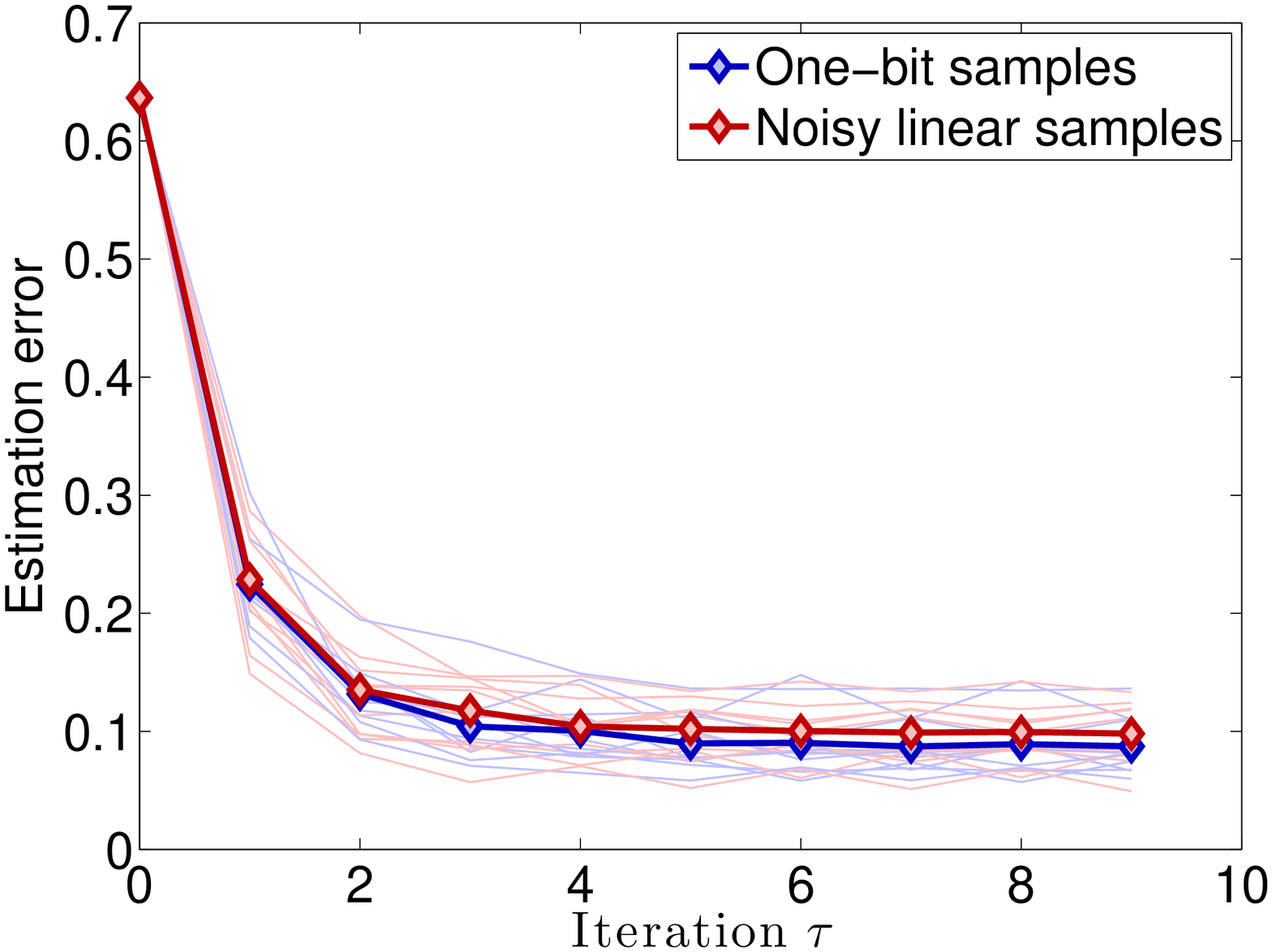}
\caption{ }
\label{lab1}
\end{subfigure}~~~
\begin{subfigure}[t]{0.45\textwidth}
\centering
\includegraphics[width=\linewidth]{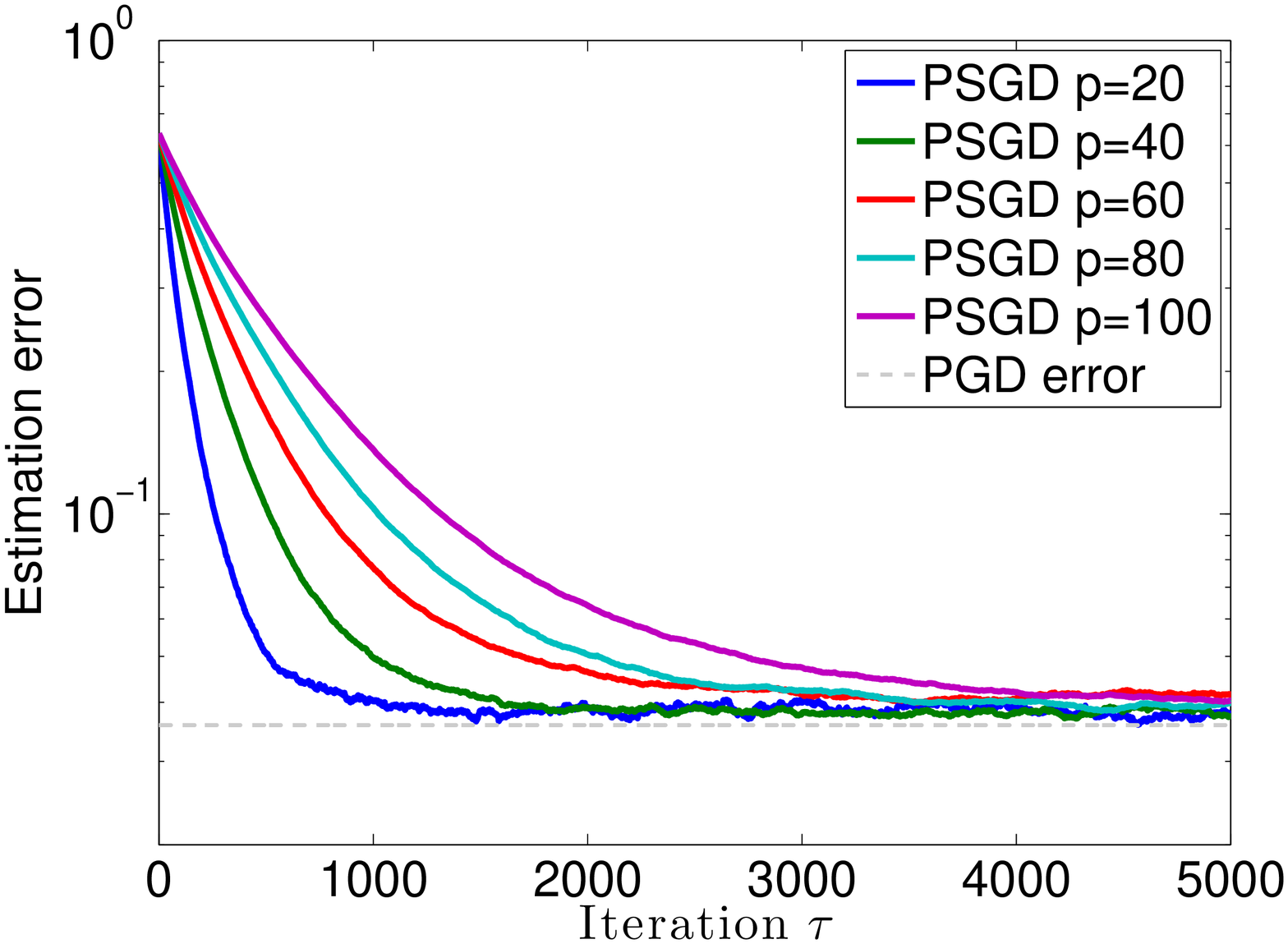}
\caption{}
\label{lab4}
\end{subfigure}
\caption{(a) Estimation errors ($\|\vct{\theta}_\tau-\sqrt{2/\pi}\vct{\theta}^*\|_{\ell_2}$) obtained via running PGD iterates as a function of the number of iterations $\tau$. The plots are for two different observations models: 1) nonlinear one bit measurements $\vct{y}=\mtx{X}\vct{\theta}^*$ and 2) noisy linear observations $\vct{y}=\sqrt{2/\pi}\mtx{X}\vct{\theta}^*+\vct{w}$ with $\vct{w}\sim\mathcal{N}(0,(1-2/\pi)\mtx{I}_n)$. The bold colors depict average behavior over $100$ trials. None bold color depict the estimation error of some sample trials. (b) Convergence behavior of PSGD for $\frac{n}{p}=4$ and $\frac{s}{p}=0.1$ as $p$ varies.}
\end{figure}
In this section we will discuss a few synthetic numerical experiments to corroborate our theoretical results in the previous sections as well as understand some of their limitations.  First we will start with some synthetic experiments followed by some experiments on natural images.
\subsection{Synthetic experiments}
In our first experiment we generate a unit norm sparse vector $\vct{\theta}^*\in\R^p$ of dimension $p=500$ containing $s=p/50$ non-zero entries. We also generate a random feature matrix $\mtx{X}\in\R^{n\times p}$ with $n=p/2$ and containing i.i.d.~$\mathcal{N}(0,1)$ entries. We now take two sets of observations of size $n$ from $\vct{\theta}^*$:
\begin{itemize}
\item One-bit nonlinear observations: the response vector is equal to $\y=\text{sgn}(\X\bbeta)$.
\item Linear observations: the response is $\y=\sqrt{\frac{2}{\pi}}\X\bbeta+\w$ with $\w\sim \Nn(0,(1-2/\pi)\Iden_n)$.
\end{itemize}
We wish to infer the vector $\sqrt{2/\pi}\vct{\theta}^*$ from these set of observations. Note that while the observation models are different, the effective noise levels in both problems are roughly of the same size i.e.~the Euclidean norm of $\vct{w}=\vct{y}-\sqrt{\frac{2}{\pi}}\mtx{X}\vct{\theta}^*$ in both cases is roughly of size $\sqrt{n}\sigma$ with $\sigma=\sqrt{1-\frac{2}{\pi}}$. We apply the PGD iterations \eqref{PGDiter} to both observations models starting from $\vct{\theta}_0=\vct{0}$. In Figure \ref{lab1} the resulting estimation errors ($\|\vct{\theta}_\tau-\sqrt{2/\pi}\vct{\theta}^*\|_{\ell_2}$) are depicted as a function of the number of iterations $\tau$. These bold colors depicts average behavior over $100$ trials. The estimation error of some sample trials are also depicted in none bold colors. The plots in Figure \ref{lab1} clearly show that PGD iterates applied to nonlinear observations converge quickly to an estimate which is of the same size as the effective noise induced by the nonlinearity. In this sense, PGD iterates converge quickly to a reliable solution which has exactly the same quality as the optimal results obtained in \cite{thrampoulidis2015lasso,plan2015generalized}.\footnote{We note that these papers do not provide any computational convergence guarantees.}  Figure \ref{lab1} also clearly demonstrates that the behavior of the PGD iterates applied to both models are essentially the same further corroborating the results of Theorem \ref{PGDthm}. This leads to the striking conclusion that there is essentially no difference between the convergence of linear and nonlinear samples when the effective noise is of the same size! In addition Figure \ref{lab1} shows that one iteration of PGD updates applied to nonlinear observations is not sufficient for reaching a statistically reliable solution and further iterations lead to further improvements. This demonstrates the advantage of our framework over other computational methods\cite{plan2014high, yi2015optimal} as the error bounds obtained in these papers are on par with the error bounds obtained by applying the first iteration of PGD.

In the next experiment we again consider the nonlinear one-bit observation model $\vct{y}=\sgn{\mtx{X}\vct{\theta}^*}$ with $\vct{\theta}^*$ a sparse vector with $s$ nonzero entries. In this experiment we fix the quantities $\frac{n}{p}=4$, $\frac{s}{p}=0.1$ (which in turn fixes $\frac{n_0}{n}$) and vary $p$. We apply the PSGD iteration of \eqref{PSGDiter} and depict the estimation error as a function of the number of iterations $\tau$ in Figure \ref{lab4} for different values of $p$. These plots show that after PSGD converges, the estimation error is the same as that of PGD. This is not consistent with the second term in \eqref{PSGDconveq} of Theorem \ref{PSGDthm}. As mentioned earlier we conjecture that the result of \eqref{PSGDconveq} holds with the second term divided by $\sqrt{n/n_0}$. Such a result would be consistent with the numerical simulation of Figure \ref{lab4}.

\subsection{Experiments on images}
\begin{figure}
\centering
\begin{subfigure}[t]{0.45\textwidth}
\centering
\includegraphics[width=\linewidth]{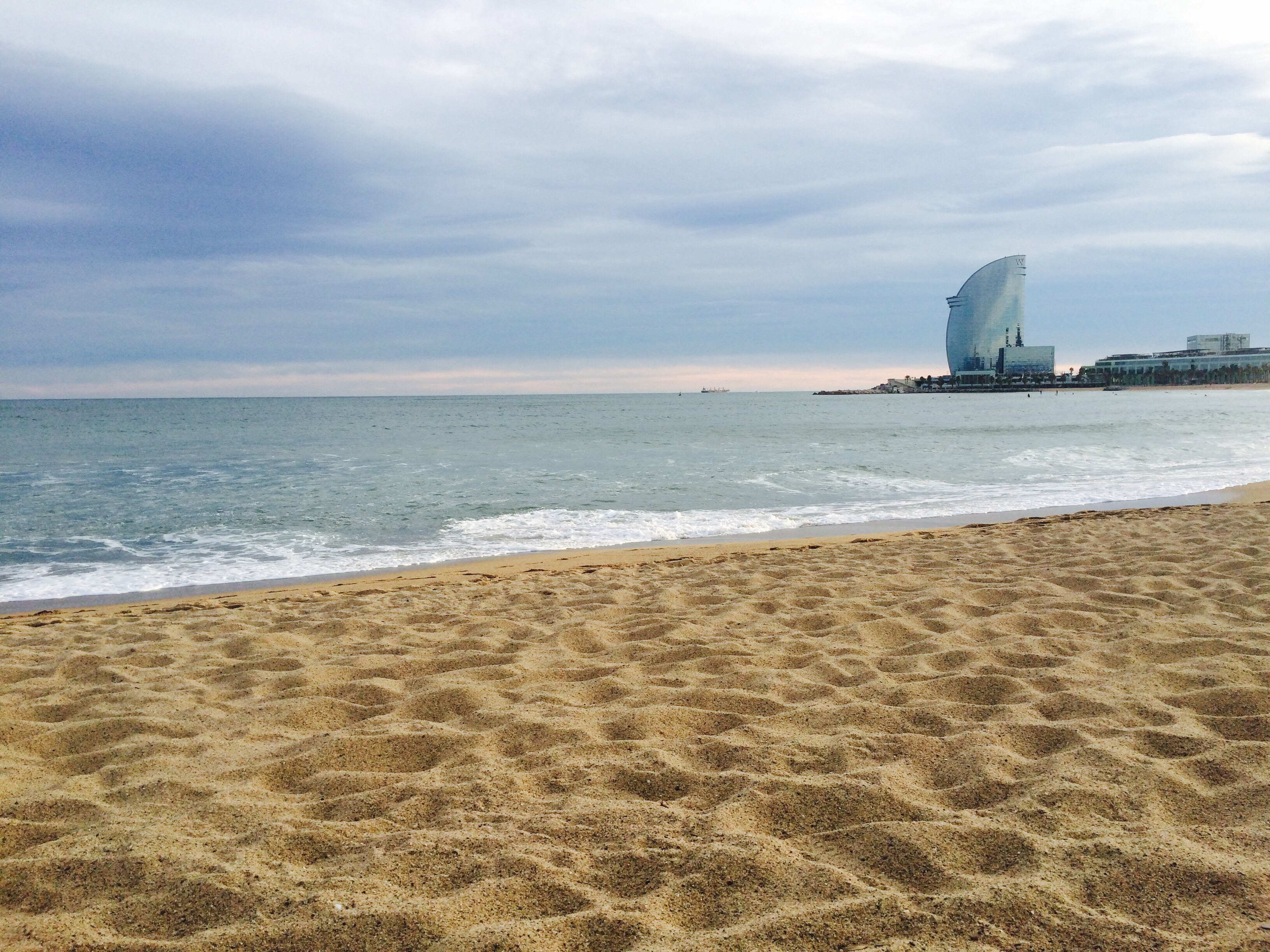}
\caption{Original image}
\end{subfigure}~~~
\begin{subfigure}[t]{0.45\textwidth}
\includegraphics[width=\linewidth]{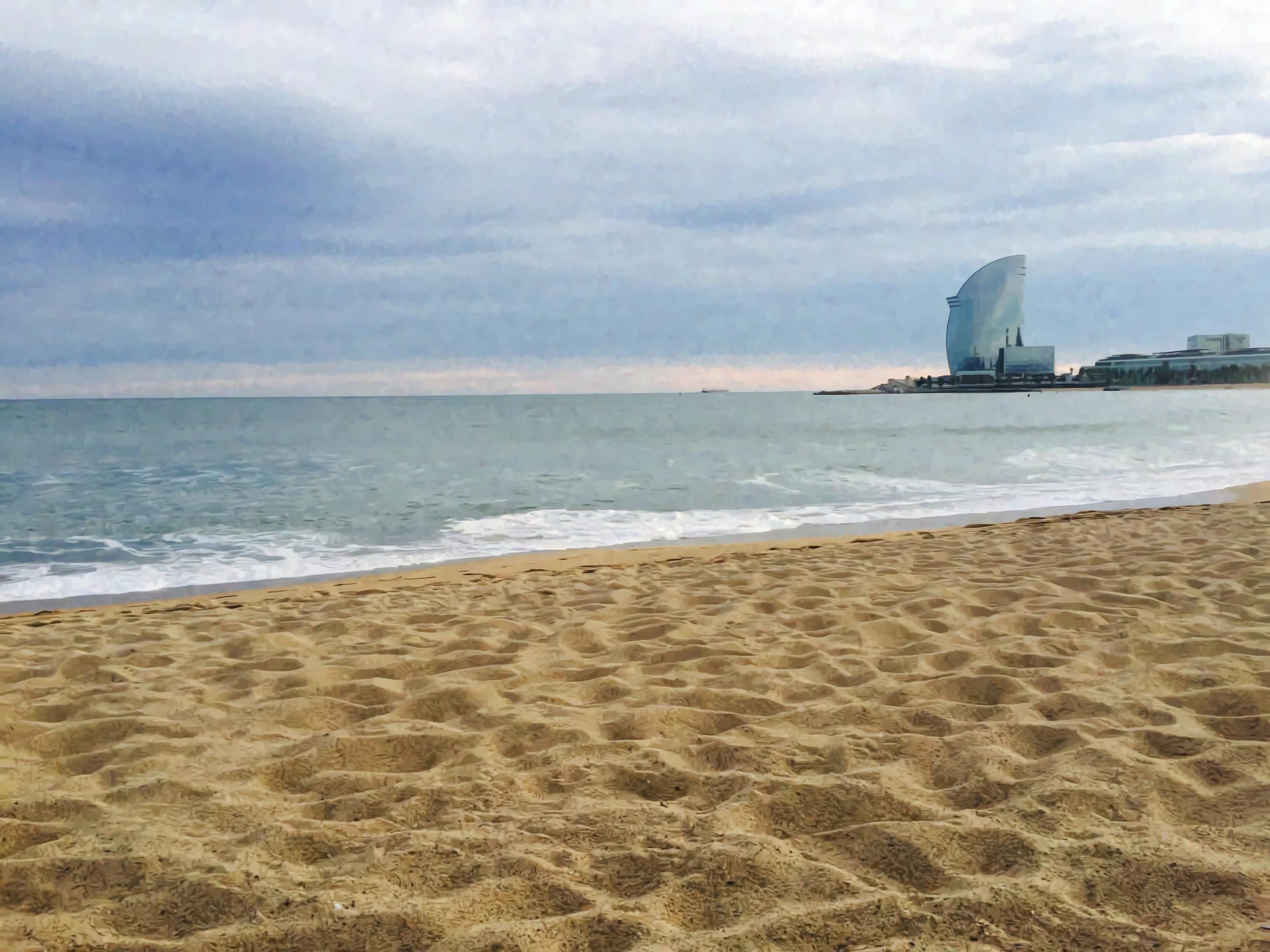}
\centering
\caption{Reconstructed image}
\end{subfigure}
\caption{(a) Image of a beach in Barcelona, Spain. (b) Reconstructed image obtained from nonlinear observations (quantized (DCT) coefficients to $16$ levels and then subsampled the observations by a factor of two). Reconstructed PSNR=$27.5944$.}
\label{natim}
\end{figure}
In this section we demonstrate the utility of an image denoiser for image recovery from quantized nonlinear observations. Our nonlinear observations consists of modulating each pixel of the image by a random i.i.d. $\pm 1$ mask, applying a two dimensional Discrete Fourier Transform (DCT), and then picking a random subset of size $m$ of these DCT coefficients and then quantizing the results to $16$ different values ($4$ bits). Since we have a color photograph we apply such nonlinear observation to each color band for a total of $3m$ nonlinear observations. 

%We shall use the short-hand $\mathcal{A}: \R^{3n}\rightarrow\R^{3m}$ to denote this linear measurement process, where $\mathcal{A}(\vct{x})$ takes a color image $\vct{x}\in\R^{3n}$ (with a total of $3n$ pixels) and outputs a vector $\vct{y}=\mathcal{A}(\vct{x})\in\R^{3m}$ containing $3m$ measurements. 

To recover the original image from such under-sampled nonlinear observations we start from $\vct{\theta}_0=\vct{0}$ and iteratively apply the updates in \eqref{ProxGDiter}. However, in liu of the prox function we use a nonlinear mapping $\mathcal{S}$ with a tunning parameter $\lambda_\tau$. We note that many nonlinear mapping can also be thought of as the prox of another function. We shall use the CBM3D denoiser \cite{dabov2007image} as the nonlinear mapping $\mathcal{S}$ so that $\mathcal{S}$ is the denoising procedure with a tunning parameter $\lambda_\tau$ which is tuned based on the assumed variance of the noise. We shall use $\lambda_\tau=\max(\lambda_0 \rho^\tau,\lambda_{\min})$ in our experiments. This choice is based on our theoretical results stated in Theorem \ref{ProxThm} of Appendix \ref{sec proximal} which suggest that this is a good tuning strategy. We now apply this proximal update with the CBM3D denoiser with $\lambda_0=14.43375\twonorm{\thets}\frac{\sqrt{p}}{n}$ and $\gamma=0.95$. We remind the reader that the image has $3n$ total pixels ($n$ pixels in each color band). For $m=0.5n$ we run $10$ iterations of the proximal update \eqref{ProxGDiter} and record the relative error $\twonorm{\hat{\vct{\theta}}-\mu\vct{\theta}^*}/\twonorm{\mu\vct{\theta}^*}$ (color images are viewed as a large vector). We depict the original image together with the reconstruction in Figure \ref{natim}. The relative noise induced by the nonlinearity in this case was $0.1977$ ($\twonorm{\mu\X\thets-f(\X\thets)}/\twonorm{\mu\X\thets}=0.1977$). The relative error we obtained by our reconstruction was $0.0757$ which is equivalent to a Peak Signal to Noise Ratio (PSNR) of $27.5944$. This figure indicates that even though the image is under-sampled by a factor of two and we quantize the image into four bits we get a rather good reconstruction.

%the relative error decreases as a function of the number of measurements. Furthermore, this value falls below $5\%$ for $m\ge0.3n$ and $m\ge0.4n$ for the Eram Garden and Fine nerve fibres images, suggesting that $30\%$ and $40\%$ under-sampling may be enough to approximately recover the image. Figure \ref{imdenois} shows that indeed the recovered images are good. We note that we did not expect projected gradient to recover the images exactly as the de-noiser may not be perfect in capturing the structure of a real image. The correct analogy in sparse recovery literature is that the image is not ``exactly" sparse. Rather it is only approximately sparse.

%Figure \ref{lab4} again considers PSGD updates for varying $n$ where the samples are $1$-bit i.e. $\y=\sgn{\X\bbeta}$. In this case we fix $\frac{m}{n}$ and $\frac{k}{n}$ so that $\frac{\omega}{\sqrt{m}}$ is fixed. Reconstruction error after convergence can be observed to be independent of $n$ which is consistent with the statement of Theorem \ref{noisy psgd}. This indicates that the result of Theorem \ref{noisy psgd} may be extended to nonlinear samples. The dashed-gray line is the average reconstruction error for the PGD updates after $50$ iterations for the same task. The average reconstruction error for PSGD and PGD is strikingly close which is consistent with $\frac{\omega}{\sqrt{m}}$ estimates arising from both Theorem \ref{main thm} and \ref{noisy psgd}.
%
%
%\MS{Have not rewritten after this}

\section{Prior Art}
\label{sec:priorart}
During the last decade, the problem of sparse estimation has been the focus of significant interest. This central problem in high-dimensional statistics is about estimating an unknown sparse parameter from possibly underdetermined observations. This task is often accomplished by the lasso estimator 
\beq
\hat{\vct{\theta}}=\arg\min_{\vct{\theta}} \frac{1}{2}\tn{\y-\mtx{X}\vct{\theta}}^2+\la \onenorm{\vct{\theta}},\label{lasso main}
\eeq
where the response vector $\y=\X\vct{\theta}^*+\w$ consists of noisy linear observations. In recent years major developments in the theory of sparse estimation \cite{Comp, Don1,RechtFazel, MatCom} have emerged. In particular, the main optimization \eqref{lasso main} has been generalized in multiple ways.
\begin{itemize}
\item {\bf{Design matrix:}} There is now a good understanding of the design matrices $\X$ that allow near-optimal estimation of $\vct{\theta}^*$. Researchers have characterized useful conditions such as restricted strong convexity, null space property and incoherence \cite{candesrip,negahban2009unified}.
\item {\bf{Parameter structure:}} $\vct{\theta}^*$ does not have to be sparse and with carefully choice of regularizers many other structures can be recovered\cite{Cha,McCoy} with a minimal number of samples.
\item {\bf{Response variable:}} The response vector $\y$ does not have to be a linear function of $\X\vct{\theta}^*$. Recent papers \cite{plan2013robust,yi2015optimal,thrampoulidis2015lasso, yang2015sparse, li2014gradient} allow for a much more general model $\y=f(\X\bbeta)$ where $f$ applies entrywise on the elements of $\X\bbeta$. Despite these interesting results most of the literature in this direction such as \cite{plan2013robust, thrampoulidis2015lasso} do not address computational issues or only address them for particular structures (e.g. sparsity) and algorithms such as \cite{yi2015optimal}.
\end{itemize}
In this paper, we propose and analyze iterative algorithms that solve nonlinear estimation problems of the form \eqref{nonlineq}. Our contributions can be summarized as follows.

$\bullet$ {\bf{Nonlinear Observations:}} We allow for an unknown nonlinear relationship between the unknown parameter $\vct{\theta}^*$ and the response variable $\y$ of the form $\vct{y}=f(\mtx{X}\vct{\theta}^*)$. By viewing nonlinearity as noise, we derive sharp performance guarantees for a variety of fast algorithms in terms of basic statistics of the nonlinearity of the unknown function $f$. 

$\bullet$ {\bf{Unified analysis:}} The same idea applies to several different algorithms with minor modification. As a result, we obtain convergence rates and estimation error bounds for interesting algorithms such as proximal gradient and projected stochastic methods.

$\bullet$ {\bf{Optimal error rates:}} Our analysis of the proposed algorithms is sharp and the resulting bounds are optimal up to small constants. Our bounds yield error rates comparable to what one would get from studying the properties of the optimum solution to the problem \eqref{main} in the special case where the regularizer $\mathcal{R}$ is a convex function \cite{thrampoulidis2015lasso}.\footnote{We note that compared to \cite{thrampoulidis2015lasso} several works on nonlinear estimation such as \cite{yi2015optimal,plan2014high} suffer from the fact that the resulting estimation error is nonzero even if there is no nonlinearity in the model (i.e.~$f(\mtx{X}\vct{\theta}^*)=\mtx{X}\vct{\theta}^*$) even if $n$ is sufficiently large. However, our analysis as well as that of \cite{thrampoulidis2015lasso} precisely recovers the unknown parameter in the linear setting with a minimal number of observations while yielding better or equal guarantees for highly nonlinear problems such as $1$-bit compressive sensing.} 

$\bullet${\bf{Nonconvex constraints:}} An interesting aspect of our framework is that it applies even when the constraints are nonconvex. As we mentioned our convergence results are for specific algorithms and we do not just study the properties of the optimum solution to \eqref{main} as in previous publications \cite{thrampoulidis2015lasso}. Furthermore, in contrast to \cite{thrampoulidis2015lasso} some of our theorems hold without requiring the regularizer to be convex. This is particularly important when the regularizer is nonconvex as it is not clear that the global optimum to \eqref{main} can be found via a tractable algorithm.

$\bullet${\bf{Precise tradeoffs between computational and statistical resources:}} In this paper we have provided precise convergence guarantees for a variety of nonlinear parameter estimation algorithms. Thus, our results provides precise tradeoffs between computational resources such as time/iterations and statistical resources such as data size, amount of nonlinearity, amount of available prior knowledge (through the choice of a regularizer) etc. In this vein, the result of this paper can also be seen as a generalization of \cite{oymak2015sharp,agarwal2010fast} to the nonlinear estimation setup. Indeed, in the absence of nonlinearity we can recover many of the results stated in \cite{oymak2015sharp,agarwal2010fast}. In comparison to \cite{oymak2015sharp,agarwal2010fast} our results also apply to a variety of different algorithms: proximal methods, stochastic methods, etc. We note however, that while our results can be applied to a variety of feature distributions such as those studied in \cite{oymak2015sharp,agarwal2010fast}, in this paper we have focused on Gaussian feature matrices $\X$.

\section{Conclusion and future directions}
In this paper we have presented a framework for characterizing time-data tradeoffs for a variety of nonlinear parameter estimation algorithms. Our results provide a precise understanding of the various tradeoffs involved between statistical and computational resources demonstrating that fast and reliable parameter estimation is possible from nonlinear observations. There are many interesting future direction to pursue:
\begin{itemize}
\item \textbf{Precise statistical constants:} In this paper we have shown that many iterative algorithms for nonlinear parameter estimation e.g. Projected Gradient Descent (PGD) converge to a solution which is a small constant factor away from the ``effective noise'' induced by the nonlinearity in these problems. However, Figure \ref{lab1} suggests that this constant should be exactly equal to one. Recently, Thrampoulidis et al. \cite{thrampoulidis2015lasso} rigorously argued this for the optimal solution to a convex program \cite{thrampoulidis2015lasso} when the regularization function is convex. However analyzing the properties of iterative algorithms with sharp convergence guarantees as provided in this paper proves to be more challenging. An interesting future direction is to show whether the ``precise'' error rates of PGD does indeed depend only on the ``effective noise'' without any additional constants. Furthermore, as we mentioned in the main text our results for PSGD seems to be off by a factor of $\sqrt{n}/\sqrt{n_0}$ closing this gap is a particularly important future direction.

\item \textbf{Parallel and lock free schemes:} For nonlinear parameter estimation problems parallel algorithms are efficient and more desirable specially when the design matrix is sparse. There are very interesting recent work for providing guarantees for parallel and lock-free implementations of stochastic algorithms \cite{recht2011hogwild}. An interesting future direction is to characterize time-data tradeoffs for such algorithms specialized for use in nonlinear parameter estimation problems.

\item \textbf{Proximal gradient methods:} Our guarantees for proximal gradient schemes discussed in Appendix \ref{sec proximal} require additional modeling assumption and are not as strong as our results for projected gradient methods. Given the wide use of proximal methods in practice providing optimal guarantees for proximal algorithms is an important future direction.

%\item We showed that linear convergence from underdetermined samples is possible while employing stochastic gradient descent algorithm. However, our estimation error rates are $\sqrt{m}$ factor off compared to the error rates of projected gradient descent. We showed that optimal estimation rates are achievable for noisy linear observations with lazy PSGD and numerical examples indicate that SGD can indeed achieve error rates on par with PGD. We leave optimal SGD error rates for nonlinear observations as a future work.% Numerical examples indicate that error rate for SGD is indeed worse than PGD as the problem dimension increases. A natural question is whether the projected SHD error rate of this paper is optimal or it can be improved.
\end{itemize}

\section{Proofs}
In this section we will prove all of our results. Throughout we use $\mathcal{B}\in\R^n$ to denote the unit ball of $\R^n$. We begin with stating some preliminary lemmas that we will use throughout the proofs. 
\subsection{Preliminaries}
In this section we gather some preliminary lemmas about projections onto sets and certain properties of Gaussian random matrices. Most of the results stated in this section are directly adapted from \cite{oymak2015sharp} (we only state the results for the convenience of the reader). The first one is a result concerning projection onto cones.
\begin{lemma}\label{firstconelem} Let $\Cc\subset\R^n$ be a closed cone and $\vb\in\R^n$. Then 
\begin{align}
%\tn{\vb}^2=&\tn{\vb-\mathcal{P}_\Cc(\vb)}^2+\tn{\mathcal{P}_\Cc(\vb)}^2\label{orthogonal},\\
\tn{\mathcal{P}_\Cc(\vb)}=&\sup_{\ub\in\Cc\cap\Bc^{n}} \ub^*\vb\label{lem:identity}.
\end{align}
\end{lemma}
The next lemma just states that translation preserves distances.
\begin{lemma} \label{simplem}Suppose $\mathcal{K}\subset\R^n$ is a closed set. The projection onto $\mathcal{K}$ obeys
\begin{align*}
\Pc_{\mathcal{K}}(\x+\vb)-\vct{x}=\Pc_{\mathcal{K}-\{\x\}}(\vb).
\end{align*}
%In particular when $\vct{x}\in\mathcal{K}$ we have
%\begin{align*}
%\Pc_{\mathcal{K}}(\x+\vb)-\x=\Pc_{\mathcal{K}-\{\x\}}(\vb).
%\end{align*}
\end{lemma}
The next lemma compares the length of a projection onto a set to the length of projection onto the conic approximation of the set. 
\begin{lemma} [Comparison of projections] \label{prop compare} Let $\mathcal{D}$ be a closed and nonempty set that contains $\vct{0}$. Let $\Cc$ be a nonempty and closed cone containing $\mathcal{D}$ ($\mathcal{D}\subset\mathcal{C}$). Then for all $\vct{v}\in\R^n$,
\begin{align}
\label{firstsetcomp}
\tn{\Pc_\Dc(\vb)}\le 2\tn{\Pc_\Cc(\vb)}
\end{align}
Furthermore, if $\Dc$ is a convex set. Then for all $\vb\in\R^n$, 
\begin{align}
\label{cvxsetcomp}
\tn{\Pc_\Dc(\vb)}\le \tn{\Pc_\Cc(\vb)}.
\end{align}
\end{lemma}
We next state a result about control of set restricted eigenvalues of Gaussian random matrices from \cite{oymak2015sharp}.
%%% beginning of special part
%%%\begin{lemma} \label{guesstimate}Let $S$ be a linear subspace, $\A$ be an i.i.d. $\Nn(0,1)$ matrix and $\Cc$ be a cone. Then, we have that
%%%\beq
%%%\E[\sup_{\vb\in\Cc\cap\Bc^{n-1}}\vb^T\Pc_S(\A^T\w)]\leq \E[\sup_{\vb\in\Cc\cap\Bc^{n-1}}\vb^T\A^T\w]\nn
%%%\eeq
%%%\end{lemma}
%%%\begin{proof} Let $\vb_S=\arg\sup_{\vb\in\Cc\cap\Bc^{n-1}}\vb^T\Pc_S(\A^T\w)$. $\vb_S$ is independent of $\Pc_S(\A^T)$ hence 
%%%\beq
%%%\E_{S^\perp}[\arg\sup_{\vb\in\Cc\cap\Bc^{n-1}}\vb^T\A^T\w]\geq \E_{S^\perp}[\vb_S^T\A^T\w]= \vb_S^T\Pc_S(\A^T\w).\nn
%%%\eeq
%%%Taking expectations over $\Pc_S(\A)$ we can conclude.
%%%\end{proof}
%%%
%%%%\section{Properties of the Gaussian matrices}
%%%\begin{lemma}[e.g. \cite{oymak2015sharp}]\label{GTtypelem} Let $\mathcal{T}\subset\R^n$ and %define
%%%%\begin{align*}
%%%%b_m=\E[\twonorm{\vct{g}}],
%%%%\end{align*}
%%%%where $\vct{g}\in\R^m$ is distributed as $\mathcal{N}(\vct{0},\mtx{I}_m)$. Furthermore, 
%%%define
%%%\begin{align*}
%%%\sigma(\mathcal{T}):=\underset{\vct{v}\in\mathcal{T}}{\max}\text{ }\twonorm{\vct{v}}.
%%%\end{align*}
%%%Then given $\eta\geq 1$
%%%\begin{align*}
%%%|\twonorm{\mtx{A}\vct{u}}-\sqrt{m}\twonorm{\vct{u}}|\le \omega(\mathcal{T})+\eta,
%%%\end{align*}
%%%holds for all $\vct{u}\in\mathcal{T}$ with probability at least 
%%%\begin{align*}
%%%1-4e^{-\frac{(\eta-1)^2}{8\sigma^2(\mathcal{T})}}.
%%%\end{align*}
%%%\end{lemma}
%%end of special part
\begin{lemma}\label{lemIAA} Let $\mathcal{C}=\mathcal{C}_{\mathcal{R}}(\mu\thets)$ be the cone of descent of the regularizer $\mathcal{R}$ at the point $\mu\thets$ as per Definition \ref{decsetcone}. Furthermore, let $\vct{n}_0=\mathcal{M}(\mathcal{R},\mu\vct{\theta}^*,t)$ be the minimal number of data samples as per Definition \ref{PTcurve}. Then for a random matrix $\mtx{X}\in\R^{n\times p}$ with i.i.d.~$\mathcal{N}(0,1)$ entries
\begin{align*}
\sup_{\vb,\ub\in \mathcal{C}\cap\Bc^{p}}\vct{u}^T\left(\mtx{I}-\frac{1}{b_n^2}\mtx{X}^T\mtx{X}\right)\vct{v}\leq \sqrt{8\frac{n_0}{n}},
\end{align*}
holds with probability at least $1-9e^{-\frac{t^2}{8}}$. Here, $b_n=\sqrt{2}\frac{\Gamma(\frac{n+1}{2})}{\Gamma(\frac{n}{2})}\approx \sqrt{n}$ with $\Gamma$ denoting the Gamma function.
\end{lemma}
The next lemma due to Gordon \cite{Gor} provides a lower bound on the minimum eigenvalue of a random Gaussian matrix restricted to a cone.
\begin{lemma}[Gordon's escape through the mesh \cite{Gor}]\label{Gordon} Assume the same setup and definitions as Lemma \ref{lemIAA} above. Then for a random matrix $\mtx{X}\in\R^{n\times p}$ with i.i.d.~$\mathcal{N}(0,1)$ entries,
\begin{align*}
\underset{\vct{u}\in\mathcal{C}\cap\mathcal{B}^p}{\inf}\twonorm{\mtx{X}\vct{u}}^2\ge (b_n-b_0)^2,
\end{align*}
holds with probability at least $1-e^{-\frac{t^2}{2}}$.
\end{lemma}
\subsection{Key Lemma for controlling nonlinear terms}
In this section we state and prove a key lemma that is crucial in our analysis and allows us to deal with the nonlinearity of our observations.
\begin{lemma} [Controlling the effective noise]  \label{lemma effective} Let $\mathcal{C}=\mathcal{C}_{\mathcal{R}}(\mu\thets)$ be the cone of descent of the regularizer $\mathcal{R}$ at the point $\mu\thets$ as per Definition \ref{decsetcone}. Also let $\mtx{X}\in\R^{n\times p}$ be a matrix of i.i.d.~$\mathcal{N}(0,1)$ entries independent of $\thets$ and $\w=f(\X\thets)-\mu\X\thets$ be the effective noise. Furthermore, let $\vct{n}_0=\mathcal{M}(\mathcal{R},\mu\vct{\theta}^*,t)$ be the minimal number of data samples as per Definition \ref{PTcurve}. Also let $\mu, \sigma$, and $\eta$ be the nonlinearity parameters for the function $f$ as per definition \ref{nonlin}. Then
\begin{align*}
\twonorm{\mathcal{P}_{\mathcal{C}}\left(\mtx{X}^T\vct{w}\right)}\le \frac{b_n^2}{\sqrt{n}}\eta\left(\sigma \sqrt{n_0}+\gamma\right).
\end{align*}
holds with probability at least $1-p(\eta)-\exp(-t^2/2)$. Here, $p(\eta)$ is the concentration probability function as per Definition \ref{defprob} and $b_n=\frac{\Gamma\left(\frac{n+1}{2}\right)}{\Gamma\left(\frac{n}{2}\right)}\approx\sqrt{n}$.
\end{lemma}
\begin{proof} We begin by defining $\mtx{X}_{||}=\mtx{X}\thets{\thets}^T$ and $\mtx{X}_{\perp}=\mtx{X}\left(\mtx{I}-\thets{\thets}^T\right)$. Based on these definition we can decompose $\mtx{X}$ into the sum of two matrices $\mtx{X}=\mtx{X}_{||}+\mtx{X}_{\perp}$ where the rows of $\mtx{X}_{||}/\mtx{X}_{\perp}$ are parallel/orthogonal to the direction of $\thets$. Now note that since $\X$ has independent standard normal entries $\X_{||}$ and $\Xp$ are independent. Hence using Lemma \ref{firstconelem}
\begin{align}
\label{twoterms}
\twonorm{\mathcal{P}_{\mathcal{C}}\left(\mtx{X}^T\vct{w}\right)}=&\underset{\vct{v}\in\mathcal{C}\cap\mathcal{B}^p}{\sup} \vct{v}^T\X^T\vct{w},\nonumber\\
=&\underset{\vct{v}\in\mathcal{C}\cap\mathcal{B}^p}{\sup} \left(\vct{v}^T\X_{||}^T\vct{w}+\vct{v}^T\X_{\perp}^T\vct{w}\right),\nonumber\\
\le& \underset{\vct{v}\in\mathcal{C}\cap\mathcal{B}^p}{\sup} \vct{v}^T\X_{||}^T\vct{w}+\underset{\vct{v}\in\mathcal{C}\cap\mathcal{B}^p}{\sup} \vct{v}^T\X_{\perp}^T\vct{w},\nonumber\\
=&\underset{\vct{v}\in\mathcal{C}\cap\mathcal{B}^p}{\sup} \vct{v}^T\thets{\thets}^T\X^T\vct{w}+\underset{\vct{v}\in\mathcal{C}\cap\mathcal{B}^p}{\sup} \vct{v}^T\X_{\perp}^T\vct{w},\nonumber\\
\le& \abs{{\thets}^T\X^T\vct{w}}\left(\underset{\vct{v}\in\mathcal{C}\cap\mathcal{B}^p}{\sup}\abs{\vct{v}^T\thets}\right)+\underset{\vct{v}\in\mathcal{C}\cap\mathcal{B}^p}{\sup} \vct{v}^T\X_{\perp}^T\vct{w},\nonumber\\
\le& \abs{{\thets}^T\X^T\vct{w}}+\underset{\vct{v}\in\mathcal{C}\cap\mathcal{B}^p}{\sup} \vct{v}^T\X_{\perp}^T\vct{w}.
\end{align}
In the last inequality we used the fact that $\abs{\vct{v}^T\thets}\le \twonorm{\vct{v}}$ together with $\twonorm{\thets}=1$. We now proceed by bounding each of these terms. To bound the first term note that ${\thets}^T\X^T\vct{w}$ can be rewritten in the form
\begin{align*}
{\thets}^T\X^T\vct{w}=(\X\thets)^T\left(f(\X\thets)-\mu\X\thets\right).
\end{align*}
Since $\X$ has i.i.d.~$\mathcal{N}(0,1)$ entries and $\twonorm{\thets}=1$, thus $\vct{g}=\X{\thets}$ is a random Gaussian vector with i.i.d.~$\mathcal{N}(0,1)$ entries. Therefore, utilizing Definitions \ref{nonlin} and \ref{defprob} 
\begin{align}
\label{firsttermnonlinlem}
\abs{{\thets}^T\X^T\vct{w}}=\abs{\vct{g}^T(f(\vct{g})-\mu\vct{g})}\le \frac{b_n^2}{\sqrt{n}}\eta\gamma,
\end{align}
holds with probability at least $1-\Pro\left(|\g^T(f(\g)-\mu\g)|>\eta \frac{b_n^2}{\sqrt{n}}\gamma\right)$.

To bound the second term in \eqref{twoterms} let $\tilde{\mathcal{C}}$ denote the projection of the set $\mathcal{C}\cap\mathcal{B}^p$ onto the plane orthogonal to the direction of $\thets$, i.e.~$\tilde{\mathcal{C}}=(\mtx{I}-\thets{\thets}^T)(\mathcal{C}\cap\mathcal{B}^p)$. Now note that for a Gaussian random vector $\vct{z}\in\R^p$ with i.i.d.~$\mathcal{N}(0,1)$ entries by standard Gaussian concentration
\begin{align}
\label{internonlinlem}
\underset{\vct{v}\in\mathcal{C}\cap\mathcal{B}^p}{\sup}\vct{v}^T(\mtx{I}-\thets{\thets}^T)\vct{z}=\underset{\vct{u}\in\tilde{\mathcal{C}}}{\sup}\text{ }\vct{u}^T\vct{z}\le \omega(\tilde{\mathcal{C}})+t\le \omega(\mathcal{C}\cap\mathcal{B}^p)+t=\omega+t,
\end{align}
holds with probability at least $1-e^{-\frac{t^2}{2}}$. Now note that $\vct{w}=f(\mtx{X}\thets)-\mu\X\thets=f(\mtx{X}_{||}\thets)-\mu\X_{||}\thets$ is only a function of $\X_{||}$ and is thus independent of $\X_{\perp}$. Thus the vector $\mtx{X}_{\perp}^T\vct{w}$ has the same distribution as a random vector $\twonorm{\vct{w}}(\mtx{I}-\thets{\thets}^T)\vct{z}$ where $\z\in\R^p$ is a Gaussian random vector that is independent of $\vct{w}$ and has i.i.d.~$\mathcal{N}(0,1)$ entries. Thus using \eqref{internonlinlem} together with Definitions \ref{nonlin} and \ref{defprob} we conclude that
\begin{align}
\label{tempnonlinlem}
\sup_{\vb\in \mathcal{C}\cap\Bc^{p}}\vb^T\Xp^T\w\leq (\omega+t)\tn{\w}=(\omega+t)\twonorm{f(\vct{g})-\mu\vct{g}}\le(\omega+t)\eta b_n\sigma, 
\end{align}
holds with probability at least $1-e^{-\frac{t^2}{2}}-\Pro(\tn{f(\g)-\mu\g}>\eta b_n\sigma)$. From Definition \ref{PTcurve} we know that $n_0$ is defined via $\omega+t=\sqrt{2}\frac{\Gamma\left(\frac{n_0+1}{2}\right)}{\Gamma\left(\frac{n_0}{2}\right)}$. By \cite[Lemma 6.9]{oymak2015sharp} $\sqrt{2}\frac{\Gamma\left(\frac{n_0+1}{2}\right)}{\Gamma\left(\frac{n_0}{2}\right)}\le b_n\sqrt{\frac{n_0}{n}}$ implying that $(\omega+t)\le b_n\sqrt{\frac{n_0}{n}}$. Plugging the latter into \eqref{tempnonlinlem} we conclude that
\begin{align}
\label{sectermnonlinlem}
\sup_{\vb\in \mathcal{C}\cap\Bc^{p}}\vb^T\Xp^T\w\le \eta b_n^2\sqrt{\frac{n_0}{n}}\sigma,
\end{align}
holds with probability at least $1-e^{-\frac{t^2}{2}}-\Pro(\tn{f(\g)-\mu\g}>\eta b_n\sigma)$. Combining \eqref{firsttermnonlinlem} and \eqref{sectermnonlinlem} via the union bound together with \eqref{twoterms} completes the proof of this lemma.
\end{proof}

%\MS{Skipping next lemma as not clear that is needed}
%\begin{lemma} \samet{not clear needed} Let $\ub,\vb$ be unit vectors obeying $\ub^T\vb=0$. Let $f:\R^n\rightarrow\R^n$ be as described in Section \ref{section setup} and $\Gb$ have i.i.d.~$\Nn(0,1)$ entries. Then, with probability $1-p(\eta)-2\exp(-(t-1)^2/2)$ the followings hold
%\begin{align}
%&\ub^T\Gb^T(f(\Gb\vb)-\mu\Gb\vb)\leq \sqrt{\eta m}t\sigma,\nn\\
%&\vb^T\Gb^T(f(\Gb\vb)-\mu\Gb\vb)\leq \sqrt{\eta m}\gamma,\nn\\
%&|1-m^{-1}\vb^T\Gb^T\Gb\vb|\leq (2t+t^2)m^{-1}.\nn
%\end{align}
%%\begin{itemize}
%%\item $\tn{f(\Gb\ub)-\mu\Gb\ub-\A\vb}\leq c\sigma\tn{\ub}+(\sqrt{m}+t)\tn{\vb}$.
%%\item 
%%\end{itemize}
%\end{lemma}
%\begin{proof}
%For unit length vectors $\vb,\ub$ obeying $\vb^T\ub=0$, conditioned on $\Gb\vb$ we have that 
%\beq
%\Pro(|\ub^T\Gb^T(f(\Gb\vb)-\Gb\vb)|\leq t \tn{f(\Gb\vb)-\Gb\vb})\geq 1-\exp(-\frac{t^2}{2}).\nn
%\eeq
%On the other hand, recalling the definition of $p(\eta)$, 
%\beq
%\Pro(\tn{f(\Gb\vb)-\Gb\vb}\geq \sqrt{\eta m}\sigma)+\Pro(\vb^T\Gb^T(f(\Gb\vb)-\mu\Gb\vb)\geq \sqrt{\eta m}\gamma)\leq p(\eta).\nn
%\eeq
%The last inequality is a standard application of Lipschitz concentration and holds with probability $1-\exp(-(t-1)^2/2)$. Union bounding these yields the result.
%% The next statement directly follows from Definition \ref{}. The final statement follows from the fact that $\Pro(|\tn{\Gb\vb}-\sqrt{m}|\leq t)\geq 1-2\exp(-\frac{t^2}{2})$.
%\end{proof}
\subsection{Proof of Theorem \ref{PGDthm}}\label{secthm1proof}
Let us denote the error in our updates by $\h_\tau=\vct{\theta}_\tau-\mu\vct{\theta}^*$ and the ``effective noise" by $\w=f(\X\vct{\theta})-\mu\X\vct{\theta}^*$. Further let $\Dc$ denote the descent set of the regularizer $\mathcal{R}$ at $\mu\vct{\theta}^*$ i.e.~$\Dc=\{\z\big| \ns(\mu\vct{\theta}^*+\z)\leq \mathcal{R}(\mu\vct{\theta}^*)\}$. Using the definition of $\vct{h}_\tau$ and $\vct{w}$ together with Lemma \ref{simplem} allows us to conclude that
\begin{align}
\h_{\tau+1}&=\Pc_{\mathcal{K}}\left(\mu\vct{\theta}^*+\left(\mtx{I}-\alpha_\tau\mtx{X}^T\mtx{X}\right)\h_\tau-\alpha_\tau\mtx{X}^T\vct{w}\right)-\mu\thets,\nn\\
&=\Pc_{\Dc}\left(\left(\mtx{I}-\alpha_\tau\mtx{X}^T\mtx{X}\right)\h_\tau-\alpha_\tau\mtx{X}^T\vct{w}\right).\nn
\end{align}
Let $\kappa_{\mathcal{R}}$ be a constant that is equal to one for convex regularizers and equal to two for nonconvex regularizers. Furthermore, assume $\mathcal{C}$ is the cone of descent of $\mathcal{R}$ at $\mu\thets$. Applying Lemmas \ref{firstconelem} and \ref{prop compare} we conclude that
\begin{align*}
\twonorm{\h_{\tau+1}}\leq &\kappa_{\mathcal{R}}\twonorm{\Pc_{\mathcal{C}}\left(\left(\mtx{I}-\alpha_\tau\mtx{X}^T\mtx{X}\right)\h_\tau-\alpha_\tau\mtx{X}^T\vct{w}\right)},\nonumber\\
\le&\kappa_{\mathcal{R}}\cdot\underset{\vct{v}\in\mathcal{C}\cap\mathcal{B}^p}{\sup}\text{ }\vct{v}^T\left(\left(\mtx{I}-\alpha_\tau\mtx{X}^T\mtx{X}\right)\h_\tau-\alpha_\tau\mtx{X}^T\vct{w}\right),\nonumber\\
\le&\kappa_{\mathcal{R}}\left(\underset{\vct{v}\in\mathcal{C}\cap\mathcal{B}^p}{\sup}\text{ }\vct{v}^T\left(\mtx{I}-\alpha_\tau\X^T\X\right)\vct{h}_\tau+\alpha_\tau\cdot\underset{\vct{v}\in\mathcal{-C}\cap\mathcal{B}^p}{\sup}\text{ }\vct{v}^T\mtx{X}^T\vct{w}\right),\\
\le&\kappa_{\mathcal{R}}\left(\twonorm{\vct{h}_\tau}\cdot\underset{\vct{u},\vct{v}\in\mathcal{C}\cap\mathcal{B}^p}{\sup}\text{ }\vct{v}^T\left(\mtx{I}-\alpha_\tau\X^T\X\right)\vct{u}+\alpha_\tau\cdot\underset{\vct{v}\in\mathcal{-C}\cap\mathcal{B}^p}{\sup}\text{ }\vct{v}^T\mtx{X}^T\vct{w}\right).
\end{align*}
Using $\alpha_\tau=1/b_n^2$ and combining Lemmas \ref{lemIAA} and \ref{lemma effective} to bound the first and second terms in the above inequality implies that
\begin{align*}
\twonorm{\vct{h}_{\tau+1}}\le \sqrt{8\kappa_{\mathcal{R}}^2\frac{n_0}{n}}\twonorm{\vct{h}_\tau}+\kappa_{\mathcal{R}}\frac{\eta\left(\sigma \sqrt{n_0}+\gamma\right)}{\sqrt{n}}.
\end{align*}
holds with probability at least $1-10e^{-\frac{t^2}{8}}-p(\eta)$ for all $\tau$. By iteratively applying the latter inequality we conclude that with high probability 
\begin{align*}
\twonorm{\vct{h}_\tau}\le&\left(\sqrt{8\kappa_{\mathcal{R}}^2\frac{n_0}{n}}\right)^\tau\twonorm{\vct{h}_0}+\kappa_{\mathcal{R}}\left(\sum_{k=0}^{\tau-1}\left(\sqrt{8\kappa_{\mathcal{R}}^2\frac{n_0}{n}}\right)^k\right)\frac{\eta\left(\sigma \sqrt{n_0}+\gamma\right)}{\sqrt{n}},\\
\le&\left(\sqrt{8\kappa_{\mathcal{R}}^2\frac{n_0}{n}}\right)^\tau\twonorm{\vct{h}_0}+\frac{\kappa_{\mathcal{R}}}{1-\left(\sqrt{8\kappa_{\mathcal{R}}^2\frac{n_0}{n}}\right)}\frac{\eta\left(\sigma \sqrt{n_0}+\gamma\right)}{\sqrt{n}},
\end{align*}
concluding the proof.
\subsection{Proof of Theorem \ref{PSGDthm}}
Our proofs in this section is in part inspired by the analysis of the randomized Kaczmarz algorithm due to Strohmer and Vershynin \cite{strohmer2009}. To begin our proof first note that by Lemma \ref{Gordon} 
\begin{align*}
\twonorm{\mtx{X}\h}^2=\sum_{i=1}^n (\vct{x}_i^*\h)^2\geq (b_n-b_{n_0})^2 \twonorm{\h}^2,
\end{align*}
holds for all $\vct{h}\in\mathcal{C}:=\mathcal{C}_{\mathcal{R}}(\mu\thets)$ with probability at least $1-e^{-\frac{t^2}{2}}$.
We now rewrite the latter equation in the alternative form
\begin{align}
\label{above eq}
\sum_{i=1}^n\frac{\twonorm{\vct{x}_i}^2}{\fronorm{\mtx{X}}^2} \left(\li\h,\frac{\vct{x}_i}{\twonorm{\vct{x}_i}}\ri\right)^2\geq \frac{(b_n-b_{n_0})^2}{\fronorm{\mtx{X}}^2}\twonorm{\vct{h}}^2.
\end{align}
Now define a random vector $\vct{z}$ distributed such that $\vct{z}=\frac{\vct{x}_i}{\twonorm{\vct{x}_i}}$ with probability $\frac{\twonorm{\vct{x}_i}^2}{\fronorm{\mtx{X}}^2}$.
Then, \eqref{above eq} can alternatively be written in the form
\begin{align}
\label{mainineq}
\E[(\vct{z}^T\vct{h})^2]\geq \frac{(b_n-b_{n_0})^2}{\fronorm{\mtx{X}}^2}\tn{\h}^2.
\end{align}
Define $\vct{h}_\tau=\vct{\theta}_\tau-\mu\vct{\theta}^*$ and recall that $\y_i=\vct{w}_i+\mu\langle\vct{x}_i,\vct{\theta}^*\rangle$. Similar to the proof of Theorem \ref{PGDthm} in Section \ref{secthm1proof} using the definition of $\vct{h}_\tau$ and $\vct{w}$ together with Lemma \ref{simplem} allows us to conclude that
\begin{align*}
\vct{h}_{\tau+1}=\mathcal{P}_{\mathcal{D}}\left(\vct{h}_\tau-\langle\vct{z}_\tau,\vct{h}_\tau\rangle\vct{z}_\tau+\frac{w_{\psi_\tau}}{\twonorm{\vct{x}_{\psi_\tau}}}\vct{z}_\tau\right),
\end{align*}
where $\{\vct{z}_\tau\}_{\tau=1}^\infty$ are independent realizations of $\vct{z}$. First, note that $\vct{0}\in\mathcal{D}$. Now we utilize the fact that projection onto a convex set containing $\vct{0}$ can only decrease the Euclidean norm of a vector we conclude that
\begin{align*}
\twonorm{\vct{h}_{\tau+1}}^2\le&\twonorm{\vct{h}_\tau-\langle\vct{z}_\tau,\vct{h}_\tau\rangle\vct{z}_\tau+\frac{w_{\psi_\tau}}{\twonorm{\vct{x}_{\psi_\tau}}}\vct{z}_\tau}^2,\\\nonumber
=&\twonorm{\vct{h}_\tau}^2-\left(\langle \vct{z}_\tau,\vct{h}_\tau\rangle\right)^2+\frac{w_{\psi_\tau}^2}{\twonorm{\vct{x}_{\psi_\tau}}^2}.
\end{align*}
Utilizing \eqref{mainineq} in the above inequality, conditioned on $\vct{h}_{\tau}$ we have
\begin{align*}
\E\big[\twonorm{\vct{h}_{\tau+1}}^2\big| \vct{h}_\tau\big]=\E\big[\twonorm{\vct{h}_\tau}^2-\left(\langle \vct{z}_\tau,\vct{h}_\tau\rangle\right)^2\big| \vct{h}_\tau\big]+\frac{1}{n}\sum_{i=1}^n\frac{w_i^2}{\twonorm{\vct{x}_i}^2}\leq \left(1-\frac{(b_n-b_{n_0})^2}{\fronorm{\mtx{X}}^2}\right)\twonorm{\vct{h}_\tau}^2+\frac{1}{n}\sum_{i=1}^n\frac{w_i^2}{\twonorm{\vct{x}_i}^2}.
\end{align*}
Using the independence of the random variables $\vct{z}_\tau$ together with law of total expectation yields 
\begin{align}
\label{interthm2}
\E\big[\twonorm{\vct{h}_{\tau+1}}^2\big]\le&\left(1-\frac{(b_n-b_{n_0})^2}{\fronorm{\mtx{X}}^2}\right)\E\big[\twonorm{\vct{h}_\tau}^2\big]+\frac{1}{n}\sum_{i=1}^n\frac{w_i^2}{\twonorm{\vct{x}_i}^2},\nonumber\\
\le&\left(1-\frac{(b_n-b_{n_0})^2}{\fronorm{\mtx{X}}^2}\right)\E\big[\twonorm{\vct{h}_\tau}^2\big]+\frac{1}{n}\frac{\twonorm{\vct{w}}^2}{\underset{i}{\min}\twonorm{\vct{x}_i}^2}.
\end{align}
Standard concentration of Chi-squared random variables together with the union bound implies that with probability at least $1-ne^{-c p}$, $\min_i\twonorm{\vct{x}_i}^2\ge 0.996p$. Here, $c$ is a fixed numerical constant. Combining the latter with \eqref{interthm2} implies that 
\begin{align*}
\E\big[\twonorm{\vct{h}_{\tau+1}}^2\big]\le \left(1-\frac{(b_n-b_{n_0})^2}{\fronorm{\mtx{X}}^2}\right)\E\big[\twonorm{\vct{h}_\tau}^2\big]+1.005\frac{\twonorm{\vct{w}}^2}{np},
\end{align*}
holds with probability at least $1-ne^{-c p}$. Using the fact that $\vct{g}=\mtx{X}\thets$ is a Gaussian random vector with $\mathcal{N}(0,1)$ entries, Definitions \ref{nonlin} and \ref{defprob} immediately imply that 
\begin{align*}
\E\big[\twonorm{\vct{h}_{\tau+1}}^2\big]\le \left(1-\frac{(b_n-b_{n_0})^2}{\fronorm{\mtx{X}}^2}\right)\E\big[\twonorm{\vct{h}_\tau}^2\big]+1.005b_n^2\frac{\eta^2\sigma^2}{np}, 
\end{align*}
holds with probability at least $1-ne^{-c p}-\Pro(\tn{f(\g)-\mu\g}>\eta b_n\sigma)$ for all $\tau$. By iteratively applying the latter inequality we conclude that with probability at least $1-ne^{-c p}-\Pro(\tn{f(\g)-\mu\g}>\eta b_n\sigma)$ we have
\begin{align}
\label{interthm2proof}
\E\big[\twonorm{\vct{h}_{\tau}}^2\big]\le& \left(1-\frac{(b_n-b_{n_0})^2}{\fronorm{\mtx{X}}^2}\right)^\tau\twonorm{\vct{h}_0}^2+1.005\eta^2b_n^2\left(\sum_{k=0}^{\tau-1}\left(1-\frac{(b_n-b_{n_0})^2}{\fronorm{\mtx{X}}^2}\right)^k\right)\frac{\eta^2\sigma^2}{np},\nonumber\\
\le& \left(1-\frac{(b_n-b_{n_0})^2}{\fronorm{\mtx{X}}^2}\right)^\tau\twonorm{\vct{h}_0}^2+1.005\frac{b_n^2}{1-\left(1-\frac{(b_n-b_{n_0})^2}{\fronorm{\mtx{X}}^2}\right)}\frac{\eta^2\sigma^2}{np},\nonumber\\
=& \left(1-b_n^2\frac{(1-\frac{b_{n_0}}{b_n})^2}{\fronorm{\mtx{X}}^2}\right)^\tau\twonorm{\vct{h}_0}^2+1.005\frac{\fronorm{\mtx{X}}^2}{\left(1-\frac{b_{n_0}}{b_n}\right)^2}\frac{\eta^2\sigma^2}{np}.
\end{align}
By \cite[Lemma 6.9]{oymak2015sharp} we have $\frac{b_{n_0}}{b_n}\le\sqrt{\frac{n_0}{n}}$. Also by standard concentration of Chi-squared random variables with probability at least $1-e^{-c np}$, $\fronorm{\mtx{X}}^2\le 1.0049np$. Furthermore, $b_n^2\ge \frac{1.0049}{2}n$ for all $n\ge 1$. Plugging the latter three inequalities into \eqref{interthm2proof} implies that
\begin{align*}
\E\big[\twonorm{\vct{h}_{\tau}}^2\big]\le \left(1-\frac{\left(1-\sqrt{\frac{n_0}{n}}\right)^2}{2p}\right)^\tau\twonorm{\vct{h}_0}^2+\frac{1.01}{\left(1-\sqrt{\frac{n_0}{n}}\right)^2}\eta^2\sigma^2,
\end{align*}
holds with probability at least $1-(n+1)e^{-cp}-p(\eta)$.

\small{
\bibliography{Bibfiles}
\bibliographystyle{plain}
}
\appendix

\section{Penalized formulation via resampling}\label{sec proximal}
In this section we aim to understand the convergence properties of a penalized variant of \eqref{lasso main} for nonlinear parameter estimation problems. Concretely, we focus on solving the following optimization problem
\begin{align}
\label{lasopt}
\min_{\theta\in\R^p}~\frac{1}{2}\twonorm{\y-\X\theta}^2+\la \ns(\theta),
\end{align}
with $\la$ a nonnegative regularization parameter. A common approach to solve the penalized formulation is via Proximal Gradient Descent (ProxGD). Starting from an initial estimate $\theta_0$, ProxGD iteratively applies the update
\begin{align}
\label{ProxGDiter}
\theta_{\tau+1}=\prox_{\la_\tau}\left(\theta_\tau+\alpha_\tau\X^T(\y-\X\theta_\tau)\right).
\end{align}
Here, $\alpha_\tau$ is the step size and $\prox_{\la}$ is the proximal function associated with $\mathcal{R}$ and is defined as
\beq
\prox_\la(\vct{z})=\arg\min_{\bar{\vct{z}}} \frac{1}{2}\tn{\vct{z}-\bar{\vct{z}}}^2+\la \ns(\bar{\vct{z}}).\nn
\eeq
In this section we wish to understand the properties of the proximal gradient iterations \eqref{ProxGDiter} for nonlinear parameter estimation problems with Gaussian features. To gain some insights into the performance of the update \eqref{ProxGDiter} we study a variant of this update where in each iteration we use a fresh set of observations and feature vectors.  In particular, we assume that we run the updates
 \begin{align}
 \label{proxresamp}
\theta_{\tau+1}=\prox_{\la_\tau}\left(\theta_\tau+\alpha_\tau\X_\tau^T(\y_\tau-\X_\tau\theta_\tau)\right).
 \end{align}
Here $\{\mtx{X}_\tau\}_{\tau=1}^\infty$ are i.i.d.~mini-batches of Gaussian features with $\mtx{X}_\tau\in\R^{n\times p}$ and $\vct{y}_\tau=f(\mtx{X}_\tau\theta)$ are mini-batches of nonlinear observations.
We emphasize that such an approach is not useful in practice as one often wishes to reuse the measurements and samples across all iterations as in the update \eqref{ProxGDiter}. Nevertheless, we hope that such an analysis provides useful insights into the performance of \eqref{ProxGDiter} and the key parameters involved in its convergence.

When dealing with the penalized version of the problem the definition of minimal number of samples as discussed in Definition \ref{PTcurve} is no longer adequate as this minimal number of samples would not only depend on the regularization function $\mathcal{R}$ but also the regularization parameter $\lambda$ in \eqref{lasopt}. In this case the minimal sample complexity is no longer based on the notion of Gaussian width but rather a closely related quantity of Gaussian distance defined below. 
\begin{definition}[Gaussian distance]\label{Gaussdist} Let $\vct{g}\in\R^p$ be a random Gaussian vector with i.i.d.~$\Nn(0,1)$ entries. Also assume $\mathcal{R}:\R^p\rightarrow \R$ is a regularization function with closed sub-level sets. For a regularization function $\mathcal{R}$ at a point $\thets$ we define the Gaussian distance at level $\lambda$ as 
\begin{align*}
\mathcal{G}(\mathcal{R},\vct{\theta}^*,\lambda):=\sqrt{\E[\emph{dist}^2(\g,\lambda\partial \mathcal{R}(\thets))]}.
\end{align*}
Here, $\partial \mathcal{R}(\thets)$ is the sub-differential of $\mathcal{R}$ at $\thets$. Also for a vector $\vct{z}\in\R^p$ and a set $\mathcal{C}\in\R^p$ the distance function $\text{dist}(\vct{z}, \mathcal{C})$ is the Eucleadian distance between the point $\vct{z}$ and the set $\mathcal{C}$ i.e.~$\text{dist}(\vct{z}, \mathcal{C})=\inf_{\bar{\vct{z}}\in \mathcal{C}} \tn{\vct{z}-\bar{\vct{z}}}$. We will often use the shorthand $\mathcal{G}(\lambda)$ with the dependence on $\mathcal{R}$ and $\thets$ implied.
\end{definition} 
The Gaussian distance defined above is closely related to the notion of mean width. In fact one can show that $\omega(\mathcal{C}_{\mathcal{R}}(\thets)\cap\mathcal{B}^p)\approx \min_{\lambda} \mathcal{G}(\mathcal{R},\thets,\lambda)$ \cite{McCoy, Foygel}. With this definition in place we are ready to explain the minimal sample complexity for the regularized case.
\begin{definition}[minimal number of samples with regularization]\label{PTcurveReg}
Let $\partial\mathcal{R}(\vct{\theta}^*)$ be the sub-differential of the regularization function $\mathcal{R}$ at $\vct{\theta}^*$ and set $\mathcal{G}(\lambda)=\mathcal{G}(\mathcal{R},\vct{\theta}^*,\lambda)$. We define the minimal sample function as
\begin{align*}
\mathcal{M}_{\lambda}(\mathcal{R},\vct{\theta}^*,t)= \phi^{-1}(\mathcal{G}(\lambda)+7t+\sqrt{2})\approx(\mathcal{G}(\lambda)+7t+\sqrt{2})^2.
\end{align*}
We shall often use the short hand $\vct{n}_0(\lambda)=\mathcal{M}_{\lambda}(\mathcal{R},\vct{\theta}^*,t)$ with the dependence on $\mathcal{R},\vct{\theta}^*,t$ implied. We note that for convex functions $\mathcal{R}$ based on \cite{OymLin} $\vct{n}_0(\lambda)$ is exactly the minimum number of samples required for the estimator \eqref{main} to succeed in recovering an unknown parameter $\vct{\theta}^*$ with high probability from linear measurements $\vct{y}=\mtx{X}\vct{\theta}^*$.% With some overloading, even for non-convex functions $\mathcal{R}$, we shall refer to $\vct{n}_0(\lambda)$ as the ``minimum number of samples".   
\end{definition}
We note that the regularized variant of the minimal sample complexity $\vct{n}_0(\lambda)$ is intimately related to the version without regularization $n_0$. In fact $n_0\approx \min_{\lambda} \vct{n}_0(\lambda)$. 

Before we state our result for the proximal iterations we would like to mention that the proximal gradient algorithm will not work if we set the shrinkage parameter $\lambda_\tau$ to be a constant in each iteration rather $\lambda_\tau$ should decay with the iterations.\footnote{The reason for this will become clear to the reader by consider the case when $f$ is a linear function and there is no noise.} In this paper we recommend a particular strategy for updating $\lambda_\tau$. Starting with tuning parameters $M_0$ and $\rho$ we set the shrinkage parameters $\lambda_\tau$ via the following set of recursions
\begin{align}
\label{shrinkageupdate}
\la_{\tau}=\frac{((1+\frac{t}{b_n})M_{\tau}+\eta\sigma)\la}{b_n}\quad\text{where}\quad M_{\tau+1}=\rho M_\tau+\frac{{\eta}(\sigma \sqrt{\vct{n}_0(\lambda)}+\gamma)}{\sqrt{n}}.
\end{align}
Observe that $M_\tau$ satisfies the bound
\beq
M_\tau\leq \rho^\tau M_0+\frac{{\eta}(\sigma \sqrt{\vct{n}_0(\lambda)}+\gamma)}{\sqrt{n}}\label{exponential bound}.
\eeq
We note that such tuning strategies our quite common in the optimization literature. In particular a related tuning strategy is utilized in the updates of the AMP algorithm \cite{Mon}. 

Now that we have described our strategy for tuning the shrinkage parameter we are ready to state our main result for the proximal gradient scheme. %The proof is more involved and beyond the scope of this paper. We defer the proof to a future publication. 
%To estimate $\bbeta$ perfectly, we make use of the homotopy method which recursively sets updates $\la$. The next proposition summarizes our result.

\begin{theorem}\label{ProxThm} Let $f:\R\rightarrow \R$ be an unknown nonlinear function. Also, for $\tau=1,2,\ldots$ let $\vct{y}_\tau=f(\mtx{X}_\tau\vct{\theta}^*)\in\R^n$ be $n$ nonlinear observations from $\vct{\theta}^*$ with the feature matrix $\mtx{X}_\tau\in\R^{n\times p}$ consisting of i.i.d.~$\mathcal{N}(0,1)$ entries. Furthermore, let $\mathcal{R}$ be a convex regularizer and let $\vct{n}_0(\lambda)=\mathcal{M}_{\lambda}(\mathcal{R},\mu\vct{\theta}^*,t)$ be the minimal number of data samples as per Definition \ref{PTcurveReg} and assume $0\leq t\leq b_n$. Assume that 
\begin{align*}
n>\vct{n}_0(\lambda).
\end{align*}
Also let $\mu, \sigma$, and $\eta$ be the nonlinearity parameters per definition \ref{nonlin} and let $\lambda$ be the regularization parameter in \eqref{lasopt}. Assume we start from an initial estimate and utilize the tuning strategy discussed in \eqref{shrinkageupdate} with tuning parameters obeying the following conditions
\begin{align*}
M_0\ge \twonorm{\vct{\theta}_0-\mu\thets}\quad\text{and}\quad\rho\ge\sqrt{\frac{\vct{n}_0(\lambda)}{n}}.
\end{align*}
Then, the proximal gradient iterations \eqref{proxresamp} with step size $\alpha_\tau=1/b_n^2\approx 1/n$ obeys
\begin{align}
%\label{PGDconveq}
\twonorm{\vct{\theta}_\tau-\mu\vct{\theta}_0}\le M_\tau,%\rho^\tau\twonorm{\vct{\theta}_0-\mu\vct{\theta}^*}+\frac{1}{\left(1-\rho\right)}\frac{\eta\left(\sigma \sqrt{n_0}+\gamma\right)}{\sqrt{n}},
\end{align}
for all $\tau$ with probability at least $1-\tau\left(2p(\eta)+7\emph{exp}(-t^2/2)\right)$.~Here, $p(\eta)$ is the concentration probability function as per Definition \ref{defprob}.
\end{theorem}
Theorem \ref{ProxThm} can be connected to our main results, for example Theorem \ref{PGDthm}. Observe that exponentially decaying error bounds \eqref{exponential bound} and \eqref{PGDconveq} have essentially identical forms. In particular, as discussed previously if $\lambda$ is chosen to be the minimizer of $\vct{n}_0(\lambda)$ then $n_0\approx \vct{n}_0(\lambda_{\min})$ and if we set $\rho=\sqrt{n_0/n}$ in the proximal iterations the convergence results of the two theorem are the same up to constant factors. In fact, the convergence rate of the theorem above is sharper and has a constant of one. This is due to the fact that we can provide sharper constants with the resampling framework. We again emphasize that while we make use of a resampling strategy such an approach is not practical and one usually wishes to run the update \eqref{ProxGDiter}. The reader is referred to \cite{xiao2013proximal,eghbali2015decomposable,AMPmain,BayMon} for related works in this direction. In particular, Xiao and Zhang \cite{xiao2013proximal} study the particular case of $\ns(\vct{x})=\onenorm{\vct{x}}$ and very recently, Erghbali and Fazel \cite{eghbali2015decomposable} have more general results that apply to the case where $\ns(\vct{x})$ is a decomposable norm and $\la$ is sufficiently large. We note that in contrast with our results these publications do not provide sharp constants with the exception of the AMP algorithm \cite{AMPmain,BayMon,metzler2014denoising}. However, rigorous results for the AMP algorithm are limited to linear estimation and separable regularizers $\cal R$.

%
%
%
%\begin{lemma} Suppose $\g\sim\Nn(0,\Iden_n)$. For any $t>0$, we have the inequality
%\beq
%\Pro(|\|\g\|^2-m|\geq 1+t\sqrt{m})\leq 2\exp(-\frac{t^2}{2}).
%\eeq
%\end{lemma}
%\begin{proof} Using Lipschitzness of the $\ell_2$ norm, we have
%\beq
%\Pro(|\|\g\|-\E[\|\g\|]|\geq t)\leq 2\exp(-\frac{t^2}{2}).
%\eeq
%Now, using the fact that $m-1\leq \E[\|\g\|]^2\leq m$, if $\|\g\|^2\geq m+2t\sqrt{m}+1$ then $\|\g\|\geq \E[\|\g\|]+t$. Similarly, if $\|\g\|^2\leq m-2t\sqrt{m}-1$ then $\|\g\|\leq \E[\|\g\|]-t$. Hence $\Pro(|\|\g\|^2-m|>1+t\sqrt{m})\leq \Pro(|\|\g\|-\E[\|\g\|]|>t)$.
%\end{proof}

%Combining these two lemma, we can conclude with the following.
%\begin{lemma} Let $\wh=\prox_\la(\y)-\z_\tau$ and $\w=\z-\Pi(\z,\la\paf)$. We have that
%\beq
%\li\ri
%\eeq
%\end{lemma}

\section{Proofs for penalized formulation (Proof of Theorem \ref{ProxThm})}
\label{secPProxThm}
In this section we will prove our result for the penalized formulation, mainly Theorem \ref{ProxThm}. Before diving into the details of the proof of Theorem \ref{ProxThm} we state and prove a few useful lemmas in Section \ref{PrePProxThm}. We then utilize these results to complete the proof of Theorem \ref{ProxThm} in Section \ref{compPProxThm}.

\subsection{Preliminary lemmas for regularized estimation}
\label{PrePProxThm}
In this section we state and prove a few results about proximal functions and Gaussian distances.
\begin{lemma} [e.g. \cite{OymProx}]\label{lem useful to know}  Let $\prox_\la$ be the proximal function associated with $\mathcal{R}$ defined as
\begin{align*}
\prox_\la(\vct{z})=\arg\min_{\bar{\vct{z}}} \frac{1}{2}\tn{\vct{z}-\bar{\vct{z}}}^2+\la \ns(\bar{\vct{z}}).\nn
\end{align*}
Then,
\beq
\tn{\prox_\la(\x+\h)-\x}\leq \emph{dist}(\h,\la \partial\mathcal{R}(\vct{x})).\nn
\eeq
\end{lemma}
The next lemma proves a useful property about the Gaussian distance.
\begin{lemma} \label{alpha decrease}Let $\Cc\subset\R^n$ be a compact and convex set. Then for a Gaussian random vector $\vct{g}$ with i.i.d.~$\mathcal{N}(0,1)$ entries, $\E[\dist^2(\alpha\g,\Cc)]$ is a nondecreasing function of $\alpha$ on $[0,\infty)$.
\end{lemma}
\begin{proof} Since $\g$ has a symmetric distribution around $0$, it suffices to show $\E[\dist^2(\alpha\g,\Cc)]+\E[\dist^2(-\alpha\g,\Cc)]$ is nondecreasing. Differentiating the square of the distance with respect to $\alpha$, we have
\beq
\frac{\partial\E[\dist^2(\alpha\g,\Cc)]}{\partial\alpha}=2\E[\li\alpha\g-\mathcal{P}_\Cc(\alpha\g),\g\ri].\nn
\eeq
Now, observe that
\beq
\li\alpha\g-\mathcal{P}_\Cc(\alpha\g),\g\ri+\li-\alpha\g-\mathcal{P}_\Cc(-\alpha\g),-\g\ri=2\alpha\tn{\g}^2-\li\mathcal{P}_\Cc(\alpha\g)-\mathcal{P}_\Cc(-\alpha\g),\g\ri.\nn
\eeq
Since $\Cc$ is convex, $\li\mathcal{P}_\Cc(\alpha\g)-\mathcal{P}_\Cc(-\alpha\g),\g\ri\leq \tn{\mathcal{P}_\Cc(\alpha\g)-\mathcal{P}_\Cc(-\alpha\g)}\tn{\g}\leq 2\alpha\tn{\g}^2$. Hence
\beq
\frac{\partial\E[\dist^2(\alpha\g,\Cc)]}{\partial\alpha}+\frac{\partial\E[\dist^2(-\alpha\g,\Cc)]}{\partial\alpha}\geq 0,\nn
\eeq
concluding the proof.
\end{proof}
Next we state a lemma that is useful for understanding the distance and projection properties of a Gaussian vector onto a subspace.
\begin{lemma} \label{proj change}Assume $\vct{g}$ is a Gaussian random vector with i.i.d.~$\mathcal{N}(0,1)$ entries. Let $\mathcal{C}\in\R^n$ be an arbitrary set that contains the origin and let $S\in\R^n$ be a subspace of dimension $n-d$. Then the following inequalities hold
\begin{align}
\E[\emph{dist}(\mathcal{P}_S(\g),\Cc)]\leq&\E[\emph{dist}(\g,\Cc)]+\sqrt{d},\label{firstGineq}\\
\E[\sup_{\vb\in\Cc}\li\mathcal{P}_S(\g),\vb\ri]\leq&\E[\sup_{\vb\in\Cc}\li\g,\vb\ri].\label{secGineq}
\end{align}
\end{lemma}

\begin{proof} First note that $\E[\tn{\g-\mathcal{P}_S(\g)}]\le\sqrt{d}$. Consequently, by the triangular inequality for distance to sets
\beq
\E[\dist(\mathcal{P}_S(\g),\Cc)]\leq \E[\dist(\g,\Cc)]+\E[\tn{\g-\mathcal{P}_S(\g)}]\leq \E[\dist(\g,\Cc)]+\sqrt{d}.\nn
\eeq
Now let $\hat\vb=\arg\sup_{\vb\in\Cc}\li\mathcal{P}_S(\g),\vb\ri$. Since $\mathcal{P}_S(\g)$ and $\g-\mathcal{P}_S(\g)$ are independent, $\hat\vb$ and $\g-\mathcal{P}_S(\g)$ are independent as well. Consequently
\beq
\E[\sup_{\vb\in\Cc}\li\mathcal{P}_S(\g),\vb\ri]=\E[\li\mathcal{P}_S(\g),\hat\vb\ri]=\E[\li\g,\hat\vb\ri]\leq \E[\sup_{\vb\in\Cc}\li\g,\vb\ri],\nn
\eeq
concluding the proof.
\end{proof}
\subsection{Proof of Theorem \ref{ProxThm}}
\label{compPProxThm}
We first provide the convergence result for a single step of the proximal iteration in the lemma below whose proof is differed to Section \ref{proofoflassolema}. 

\begin{lemma} [Single step estimator]\label{lasso prop 1} Let $\vct{\theta}_\tau$ be the estimate at the $\tau$th iteration with the associated error $\h_\tau=\vct{\theta}_\tau-\mu\vct{\theta}^*$ obeying $\tn{\h_\tau}\leq M$ for some $M\geq 0$. Given $\la\geq 0$ and $0\leq t\leq b_n$, pick $\la_\tau=\frac{\la ((b_n+t)M+b_n\eta \sigma)}{b_n^2}$. In order to estimate $\mu\vct{\theta}^*$ from $\y=f(\X\vct{\theta}^*)$ consider the following update
\beq
\vct{\theta}_{\tau+1}=\prox_{\la_\tau}\left(\vct{\theta}_\tau+\frac{1}{b_n^2}\X^T(\y-\X\vct{\theta}_\tau)\right).\nn
\eeq
Let $\vct{n}_0(\lambda)=\mathcal{M}_{\lambda}(\mathcal{R},\mu\vct{\theta}^*,t)$ be the minimal number of data samples as per Definition \ref{PTcurveReg} and assume $0\leq t\leq b_n$.
%\beq
%\rho=\frac{\sqrt{\Glf}+\sqrt{2}+2t}{b_n}\nn
%\eeq
If $\X$ has i.i.d. $\Nn(0,1)$ entries, then with probability at least $1-2p(\eta)-7\exp(-t^2/2)$, $\vct{\theta}_{\tau+1}$ obeys
\beq
\tn{\vct{\theta}_{\tau+1}-\mu\vct{\theta}^*}\leq \sqrt{\frac{\vct{n}_0(\lambda)}{n}} M+\eta \sqrt{\frac{\vct{n}_0(\lambda)}{n}}\sigma+\eta\gamma/\sqrt{n}.\nn%\frac{\Db(\la\paf)}{\sqrt{m}}+t(\frac{1}{\sqrt{m}}+\frac{\sqrt{n}}{m})]\|\z_\tau-\bbeta\|.\nn
\eeq
\end{lemma}
With this lemma in hand we are ready to prove Theorem \ref{ProxThm}. We shall show this by induction. Suppose the residual obeys $\tn{\h_\tau}\leq M_\tau$ with probability at least $1-\tau P$ where $P=2 p(\eta)+7\exp(-\frac{t^2}{2})$. Now, observe that the particular choice of $\la_\tau$ (as a function of $M_\tau$) makes Proposition \ref{lasso prop 1} applicable.

Using the fact that new samples are independent of the rest, Lemma \ref{lasso prop 1} implies that
\begin{align}
\tn{\h_{\tau+1}}&\leq\sqrt{\frac{\vct{n}_0(\lambda)}{n}} M_\tau+\eta\sqrt{\frac{\vct{n}_0(\lambda)}{n}}\sigma +\eta\gamma/\sqrt{n}\leq M_{\tau+1}\nn,
\end{align}
holds with probability at least $1-P$. Applying the union bound, $\tn{\h_{i}}\leq M_i$ holds for all $0\leq i\leq \tau+1$ with probability at least $1-(\tau+1)P$, completing the proof. All that remains now is to complete the proof of Lemma \ref{lasso prop 1} which is the subject of the next section.

\subsubsection{Proof of Lemma \ref{lasso prop 1}}
\label{proofoflassolema}
\begin{proof} Define $\w=f(\X\vct{\theta}^*)-\mu\X\vct{\theta}^*$. The term inside the proximal operator can be rewritten as
\begin{align}
\vct{\theta}_\tau+\frac{1}{b_n^{2}}\X^T(\y-\X\vct{\theta}_\tau)&=\mu\vct{\theta}^*+\left(\h_\tau-\frac{1}{b_n^2}\X^T\X\h_\tau\right)+\frac{1}{b_n^2}\X^T(\y-\mu\X\vct{\theta}^*).\nn\\
&=\mu\vct{\theta}^*+\left(\h_\tau-\frac{1}{b_n^2}\X^T\X\h_\tau\right)+\frac{1}{b_n^2}\X^T\w.
\end{align}
Note that $\mu\vct{\theta}^*$ is the term we wish our iterates to converge to and the remaining terms can be viewed as noise. Define $\mathcal{S}$ to be the $n-2$ dimensional subspace perpendicular to $\vct{\theta}^*,\h_\tau$. Given $\vb\in \mathcal{S}^{\perp}$, let $\vb^\perp$ be the projection of $\vb$ onto the direction perpendicular to $\vb$ and $\mathcal{S}$. The noise terms will be split into three terms by using $\X^T=\Pc_{\mathcal{S}}(\X^T)+\Pc_{\vb^\perp}(\X^T)+\Pc_{\vb}(\X^T)$.
\begin{itemize}
\item $\vct{e}_1=\Pc_{\mathcal{S}}(\X^T)(-\X\h_\tau+(f(\X\vct{\theta}^*)-\mu\X\vct{\theta}^*))$.
\item $\vct{e}_2=\Pc_{\vct{\theta}^*}(\X^T)(f(\X\vct{\theta}^*)-\mu\X\vct{\theta}^*)+\Pc_{\left(\vct{\theta}^*\right)^\perp}(\X^T)(f(\X\vct{\theta}^*)-\mu\X\vct{\theta}^*)$.
\item $\vct{e}_3=b_n^2\h_\tau-\Pc_{\h_\tau}(\X^T)\X\h_\tau-\Pc_{\h_\tau^\perp}(\X^T)\X\h_\tau$.
\end{itemize}
With this notation
\beq
\vct{\theta}_{\tau+1}=\prox_{\la_\tau}\left(\mu\vct{\theta}^*+\frac{1}{b_n^2}(\vct{e}_1+\vct{e}_2+\vct{e}_3)\right).\nn
\eeq
We next relate the proximal estimator to the subdifferential via Lemma \ref{lem useful to know}. This yields
\begin{align}
\tn{\vct{\theta}_{\tau+1}-\mu\vct{\theta}^*}&\leq \dist\left(\frac{1}{b_n^2}(\vct{e}_1+\vct{e}_2+\vct{e}_3),\la_\tau\partial \ns(\mu\vct{\theta}^*)\right).\nn\\
&\leq \frac{1}{b_n^2}\left(\dist\left(\vct{e}_1,b_n^2\la_\tau\partial \ns(\mu\vct{\theta}^*)\right)+\tn{\vct{e}_2}+\tn{\vct{e}_3}\right).\nn
\end{align}
We now proceed by estimating each of these terms. We will show that $\vct{e}_2$ and $\vct{e}_3$ are fairly small and we will obtain a bound for the term involving $\vct{e}_1$.

{\bf{Estimating $\vct{e}_2$:}} To estimate $\vct{e}_2$ we use the fact that $\tn{f(\X\vct{\theta}^*)-\mu\X\vct{\theta}^*}\leq \eta b_n\sigma$ and 
\begin{align}
\label{temp1lasso}
\tn{\Pc_{\vct{\theta}^*}(\X^T)\left(f(\X\vct{\theta}^*)-\mu\X\vct{\theta}^*\right)}\leq \eta b_n^2\gamma/\sqrt{n},
\end{align}
holds with probability at least $1-p(\eta)$. Next we use the fact that ${\left(\vct{\theta}^*\right)^\perp}^T\X^T$ is an i.i.d.~normal random vector and is independent of $\mtx{X}\vct{\theta}^*$ to conclude that
\begin{align}
\label{temp2lasso}
\twonorm{\Pc_{\left(\vct{\theta}^*\right)^\perp}(\X^T)\left(f(\X\vct{\theta}^*)-\mu\X\vct{\theta}^*\right)}\leq t\sigma\eta b_n,
\end{align}
holds with probability at least $1-\exp(-t^2/2)$. Combining \eqref{temp1lasso} and \eqref{temp2lasso} together with the union bound yields
\beq
\tn{\vct{e}_2}\leq \eta b_n(t\sigma+\gamma b_n/\sqrt{n})\label{part e2}.
\eeq

{\bf{Estimating $\vct{e}_3$:}} We now bound the term involving $\vct{e}_3$. For $t\leq b_n$, with probability at least $1-3\exp(-t^2/2)-\exp(-b_n^2/2)\geq 1-4\exp(-t^2/2)$, the followings identities hold. First, using an independence argument again $\tn{\Pc_{\h_\tau^\perp}(\X^T)\X\h_\tau}\leq 2tb_n\tn{\h_\tau}$. Second, $|\tn{\X\h_\tau}-b_n\tn{\h_\tau}|\leq t\tn{\h_\tau}$. Using this, it follows that 
\beq
\tn{b_n^2\h_\tau-\Pc_{\h_\tau}(\X^T)\X\h_\tau}\leq \tn{\h_\tau}(2b_nt+t^2)\leq 3tb_n\tn{\h_\tau}.\nn
\eeq Combining these identities we arrive at
\beq
\tn{\vct{e}_3}\leq 5tb_n\tn{\h_\tau}.\label{part e3}
\eeq
Combining bounds \eqref{part e2} and \eqref{part e3} which involve $\e_2$ and $\e_3$, with probability at least $1-5\exp(-t^2/2)-p(\eta)$ we have% \samet{fix factor 5}
\beq
b_n^{-2}(\tn{\e_2}+\tn{\e_3})\leq b_n^{-1}[5t\tn{\h_\tau}+\eta(t\sigma+\gamma b_n/\sqrt{n})]\leq (\eta/b_n)[t(5\eta^{-1}\tn{\h_\tau}+\sigma)+\gamma b_n/\sqrt{n}].\label{part e23}
\eeq
We note that this bound grows as $b_n^{-1}$.
%\samet{assumes $\eta\geq 4$} 

{\bf{Estimating $\vct{e}_1$:}} The remaining term is $\vct{e}_1$. Define $\ab:=\X\h_\tau+f(\X\vct{\theta}^*)-\mu\X\vct{\theta}^*$. Observe that with probability at least $1-\exp(-t^2/2)-p(\eta)$
\beq%\leq \sqrt{\eta m}(\tn{\h_\tau}+\sigma)
\sigma_{tot}:=\tn{\ab}\leq (b_n+t)\tn{\h_\tau}+\eta b_n\sigma\leq (b_n+t)M+\eta b_n\sigma:=\sigma_{up}.\nn
\eeq 
Note that $\h_\tau,\vct{\theta}^*\in \mathcal{S}^\perp$. Thus, conditioned on $\ab$, $\Pc_{\mathcal{S}}(\X^T)\ab$ is statistically identical to $\g'=\Pc_{\mathcal{S}}(\g)$ with $\g\sim\Nn(0,\sigma_{tot}^2\Iden_n)$. Now, applying Lemma \ref{proj change}
\beq
\E[\dist(\Pc_S(\X^T)[\X\h_\tau+f(\X\vct{\theta}^*)-\mu\X\vct{\theta}^*],b_n^2\la_\tau\paf)]\leq \E[\dist(\g,b_n^2\la_\tau\paf)]+\sqrt{2}\sigma_{tot}.\label{first m1}
\eeq%\samet{left here} 
The problem is reduced to upper bounding $\E[\dist(\g,m\la_\tau\paf)]$. Using Lemma \ref{alpha decrease}%\samet{?}
\begin{align}
\E[\dist(\g,b_n^2\la_\tau\paf)]&\leq \sqrt{\E[\dist(\g,b_n^2\la_\tau\paf)^2]}\nn\\
&=\sigma_{up}\sqrt{\E[\dist(\sigma_{up}^{-1}\g,\sigma_{up}^{-1}b_n^2\la_\tau\paf)^2]}\nn\\
&=\sigma_{up}\sqrt{\E[\dist(\sigma_{up}^{-1}\g,\la\paf)^2]}\nn\\
&\leq \sigma_{up}\sqrt{\Glf}\label{two from last},
%&\leq (b_nM+\eta b_n\sigma)(\sqrt{\Glf}+t)\nn.
\end{align}
where \eqref{two from last} follows from the definition of $\la_\tau$ and Lemma \ref{alpha decrease}.  Merging this with \eqref{first m1}, together with the union bound implies that with probability at least $1-2\exp(-t^2/2)-p(\eta)$
 %and $\sigma_{tot}\leq \sigma_{up}$ and last line follows from $\Glf\leq b_n^2$. 
\begin{align}
\dist(\Pc_{\mathcal{S}}(\X^T)[\X\h_\tau+f(\X\vct{\theta}^*)-\mu\X\vct{\theta}^*],b_n^2\la_\tau\paf)&\leq \sigma_{up}(\sqrt{\Glf}+\sqrt{2}+t)\nn\\
&\leq (b_nM+\eta b_n\sigma)(\sqrt{\Glf}+\sqrt{2}+2t).\label{part e1}
\end{align}
Here, the last line follows from $b_n\geq \sqrt{\Glf}+\sqrt{2}+2t$. Merging \eqref{part e23} and \eqref{part e1} and recalling the definition of the convergence rate $\rho$ the cumulative error takes the form
\begin{align}
\tn{\vct{\theta}_\tau-\mu\vct{\theta}^*}&\leq b_n^{-1}[M(\sqrt{\Glf}+\sqrt{2}+7t)]+\eta b_n^{-1}[\gamma b_n/\sqrt{n}+\sigma(\sqrt{\Glf}+\sqrt{2}+3t)]\nn\\
&\leq b_n^{-1} M(\sqrt{\Glf}+\sqrt{2}+7t)+\eta b_n^{-1}[\gamma b_n/\sqrt{n}+\sigma(\sqrt{\Glf}+\sqrt{2}+7t)]\nn\\
&\leq \sqrt{\frac{\vct{n}_0(\lambda)}{n}} M+\eta\sqrt{\frac{\vct{n}_0(\lambda)}{n}} \sigma+\eta \gamma/\sqrt{n},\nn
\end{align}
which is the advertised error bound. %Applying a final union bound over $\e_1,\e_2,\e_3$, this bound holds with probability $1-7\exp(-t^2/2)-2p(\eta)$.
\end{proof}
\end{document}